%% file: neurips_2026.tex
\documentclass{article}

\usepackage[main,final]{neurips_2026}
\usepackage{ragged2e}
\usepackage{enumitem}
\usepackage{hyperref}

\usepackage[utf8]{inputenc}
\usepackage[T1]{fontenc}
\usepackage{hyperref}
\usepackage{url}
\usepackage{booktabs}
\usepackage{amsfonts}
\usepackage{nicefrac}
\usepackage{microtype}
\usepackage{xcolor}
\usepackage{amsmath}
\usepackage{amssymb}
\usepackage{mathtools}
\usepackage{amsthm}
\usepackage{booktabs}
\usepackage{tabularx}
\usepackage{multirow}
\usepackage{listings}

\usepackage{xcolor}
\usepackage{enumitem}
\usepackage{hyperref}
\usepackage{url}
\usepackage{mathtools}
\usepackage{amsthm}

\usepackage{bm}
\usepackage{subcaption}
\usepackage{multirow}
\usepackage{graphics}
\usepackage{float}
\usepackage{wrapfig}
\usepackage{array}
\usepackage{caption}
\usepackage[capitalize,noabbrev]{cleveref}
\usepackage{subcaption}

\usepackage{amsmath}
\usepackage{amssymb}
\usepackage{mathtools}
\usepackage{amsthm}

\usepackage{subcaption}
\usepackage{graphicx}
\usepackage{svg}
\usepackage{makecell}
\usepackage[capitalize,noabbrev]{cleveref}

\theoremstyle{plain}
\newtheorem{theorem}{Theorem}[section]
\newtheorem{proposition}[theorem]{Proposition}

\newtheorem{corollary}[theorem]{Corollary}
\theoremstyle{definition}

\theoremstyle{remark}

\usepackage[most]{tcolorbox}
\usepackage{xcolor}
\usepackage{soul}

\definecolor{badcolor}{HTML}{FFCCCB}
\definecolor{goodcolor}{HTML}{D0F0C0}
\definecolor{promptcolor}{HTML}{F0F0F0}

\newtcolorbox{promptbox}[1][]{%
    colback=promptcolor, colframe=gray!50, arc=2mm, boxrule=1pt,
    fonttitle=\bfseries, title={Prompt/Instruction}, #1}

\newtcolorbox{responsebox}[2][]{%
    colback=white, colframe=#2, arc=2mm, boxrule=1.5pt,
    fonttitle=\bfseries, title={#1}, enhanced, attach boxed title to top left={yshift=-2mm, xshift=2mm},
    boxed title style={colback=#2}}
\usepackage{microtype}
\usepackage{graphicx}
\usepackage{subcaption}
\usepackage{booktabs}
\usepackage{algorithm}
\usepackage{algpseudocode}
\title{Learning Where It Matters: Geometric \\ Anchoring for Robust Preference Alignment}

\author{%
  Youngjae Cho$^{1}$ \quad
  Jongsuk Kim$^{1}$ \quad
  Ji-Hoon Kim$^{2}$\thanks{Corresponding author. Work done while at PYLER.} \\
  $^{1}$PYLER \quad $^{2}$FINDA \\
  \texttt{\{youngjae.cho, js.kim\}@pyler.tech} \quad
  \texttt{genesis717.jhk@gmail.com}
}
\hypersetup{hidelinks}

\begin{document}

\maketitle

\begin{abstract}
Preference optimization aligns large language models from pairwise preferences by increasing the margin between preferred and dispreferred responses.
However, margin-based losses alone do not test whether each pair's preference signal is locally stable enough to trust with its update strength.
We propose \emph{Geometric Anchor Preference Optimization} (GAPO), a geometry-aware objective that introduces a batch-conditioned stress test for preference learning.
For each mini-batch, GAPO constructs a pessimistic \emph{anchor} by perturbing the current policy in the first-order direction that decreases the batch-average preference margin.
The resulting \emph{Anchor Gap} measures how much each pair's margin degrades under this shared pessimistic probe and converts this degradation into an instance-wise update weight.
Pairs with larger Anchor Gap receive smaller update weights, while non-brittle pairs retain their preference-gradient direction.
Under smoothness assumptions, we characterize the Anchor Gap as a batch-directional proxy for local margin degradation.
Across multiple open-weight model families, GAPO matches or improves strong preference-optimization baselines on instruction-following and reasoning benchmarks.
It also improves robustness under random and structured preference noise without explicitly modeling label corruption.
Mechanistic diagnostics show that the pairs receiving the smallest GAPO weights are statistically enriched in corrupted supervision, suggesting that GAPO improves robustness by reducing the cumulative influence of brittle preference signals.
\end{abstract}

\input{tex/1_intro}
\input{tex/2_rel}
\input{tex/3_pre}
\input{tex/4_method}
\input{tex/5_analysis}
\input{tex/6_exp}
\input{tex/7_con}
\clearpage
\begin{ack}
\input{tex/funding_disclosure}
\end{ack}
\bibliography{neurips_2026}

\newpage
\input{tex/appendix}
\clearpage
\input{checklist.tex}

\end{document}

%% file: tex/1_intro.tex
\section{Introduction}
\label{sec:intro}

Aligning large language models (LLMs) with human intent is a central problem in modern AI.
Modern offline alignment methods~\citep{schulman2017proximal,ouyang2022training,christiano2017deep} often learn from pairwise preferences, where each example compares a preferred response against a dispreferred response for a prompt.
Most direct preference objectives~\citep{rafailov2023direct,meng2024simpo} train the policy by increasing an implicit margin between the two responses, and this margin determines how strongly each pair contributes to the update.

However, the current margin alone is an incomplete measure of how much a pair should influence training.
A low-margin pair may be a genuinely informative hard example, but it may also reflect a brittle preference boundary induced by noisy labels, length bias, or other spurious correlations. Margin-based objectives treat these cases similarly because both can produce large updates. This motivates a local diagnostic that tests whether the pair's margin degrades sharply under a small movement around the current policy before trusting it with a large update.

We introduce \emph{Geometric Anchor Preference Optimization} (GAPO), which instantiates this diagnostic as a batch-conditioned local stress test.
Rather than judging each pair solely by the size of its current margin, GAPO asks whether the pair's margin collapses under a shared pessimistic probe induced by the current mini-batch.
This probe is used as a diagnostic, not as a requirement that useful pairs be robust to all perturbations. To instantiate this test, GAPO constructs a pessimistic local reference neighbor, or \emph{anchor}, around the current policy.
For each mini-batch, we perturb the policy along the first-order direction of steepest decrease in the batch-average preference margin within a small radius.
This anchor is not intended to predict the optimizer's  next step.
Instead, it defines a shared pessimistic local probe for the batch, allowing all pairs to be evaluated under the same adverse condition.

We then define the \emph{Anchor Gap} as the degradation of a pair's margin under this anchor.
A large Anchor Gap means that the pair's margin collapses under the batch-conditioned stress test, indicating local brittleness; GAPO down-weights its contribution.
A small Anchor Gap means that the pair remains stable under the same pessimistic condition, and GAPO preserves its standard preference-learning influence.
Crucially, GAPO does not classify pairs as good or bad, estimate an intrinsic data-quality score, or filter the dataset.
It is a risk-aware reweighting mechanism that continuously modulates each pair's contribution based on its local stability, suppressing brittle gradients while preserving non-brittle ones.

Because the anchor margin is detached, GAPO preserves each pair's
current margin-gradient direction, $\nabla_\theta M_i(\theta)$, while
changing its scalar coefficient according to the Anchor Gap. Thus, GAPO modifies how preference gradients are weighted across instances rather than introducing a gradient path through the perturbed policy.

Under smoothness assumptions, we characterize the geometric anchor as a batch-directional probe of local margin degradation: the per-instance Anchor Gap consists of a first-order component from gradient magnitude and batch-direction alignment, plus a second-order directional-curvature correction.
This decomposition helps explain why GAPO is not reducible to gradient-norm reweighting alone.

GAPO improves standard instruction-following performance while preserving reasoning ability across model families.
It also improves robustness under noisy preference supervision, including random label flips and structured length-dependent biases.
Mechanistic analyses show that, in controlled noisy settings, corrupted or systematically biased pairs concentrate in the large-Anchor-Gap and are therefore down-weighted most strongly.
These findings indicate that GAPO's robustness arises from conservative attenuation of brittle preference contributions rather than explicit noise detection.

Our contributions are as follows:
\begin{itemize}[leftmargin=*, itemsep=2pt, topsep=2pt]
  \item We propose \emph{Geometric Anchor Preference Optimization} (GAPO), which constructs a batch-conditioned pessimistic anchor and uses the resulting \emph{Anchor Gap} to conservatively attenuate locally brittle preference pairs.
  \item We provide a local theoretical analysis showing that the implemented anchor acts as a batch-directional probe of margin degradation, and that GAPO preserves each pair's preference-gradient direction up to an Anchor-Gap-based scalar reweighting.
  \item We demonstrate that GAPO matches or exceeds strong preference optimization baselines on instruction-following and reasoning benchmarks, while improving robustness under random and structured preference noise.
\end{itemize}

%% file: tex/2_rel.tex
\section{Related Work}
\label{sec:related}

\paragraph{Preference optimization and robustness.}
Offline preference optimization includes reference-based objectives such as DPO~\citep{rafailov2023direct}, IPO~\citep{azar2024general}, and KTO~\citep{ethayarajh2024kto}, as well as variants addressing length bias or weak supervision such as RDPO~\citep{park2024disentangling} and ORPO~\citep{hong2024orpo}. Reference-free methods such as CPO~\citep{xu2024contrastive} and SimPO~\citep{meng2024simpo} remove the fixed reference for efficiency, while $\alpha$-DPO~\citep{wu2025alphadpo} interpolates between the SFT model and the evolving policy. $\beta$-DPO \citep{wu2024betadpo} and $\epsilon$-DPO \citep{lee2025kl} adapt the KL coefficient $\beta$ at the batch and instance level, respectively. GAPO instead keeps $\beta$ and $\gamma$ fixed and adapts per pair gradient weights from margin degradation under a shared parameter-space perturbation \ref{eq:gradient}.
These methods expose a tension between reference-induced stability and policy flexibility. For noisy supervision, prior work typically models label reliability directly, for example through label smoothing in r-DPO~\citep{chowdhury2024provably} or distributionally robust optimization in Dr.~DPO~\citep{wu2024towards}. GAPO takes a complementary approach: rather than estimating label noise, it uses a shared pessimistic perturbation to measure local preference stability and attenuate brittle pairs.

\paragraph{Geometric perspective.}
Sharpness-Aware Minimization (SAM)~\citep{foret2020sharpness} improves generalization by perturbing parameters to seek flatter solutions. GAPO uses a related geometric idea for a different purpose: it converts local margin degradation under a batch-conditioned pessimistic anchor into an instance-wise reweighting signal, selectively downweighting brittle preference pairs while preserving the underlying preference-gradient direction.

%% file: tex/3_pre.tex
\begin{figure*}[t!]
  \centering

  \begin{subfigure}[b]{0.32\linewidth}
    \centering
    \includegraphics[width=\linewidth]{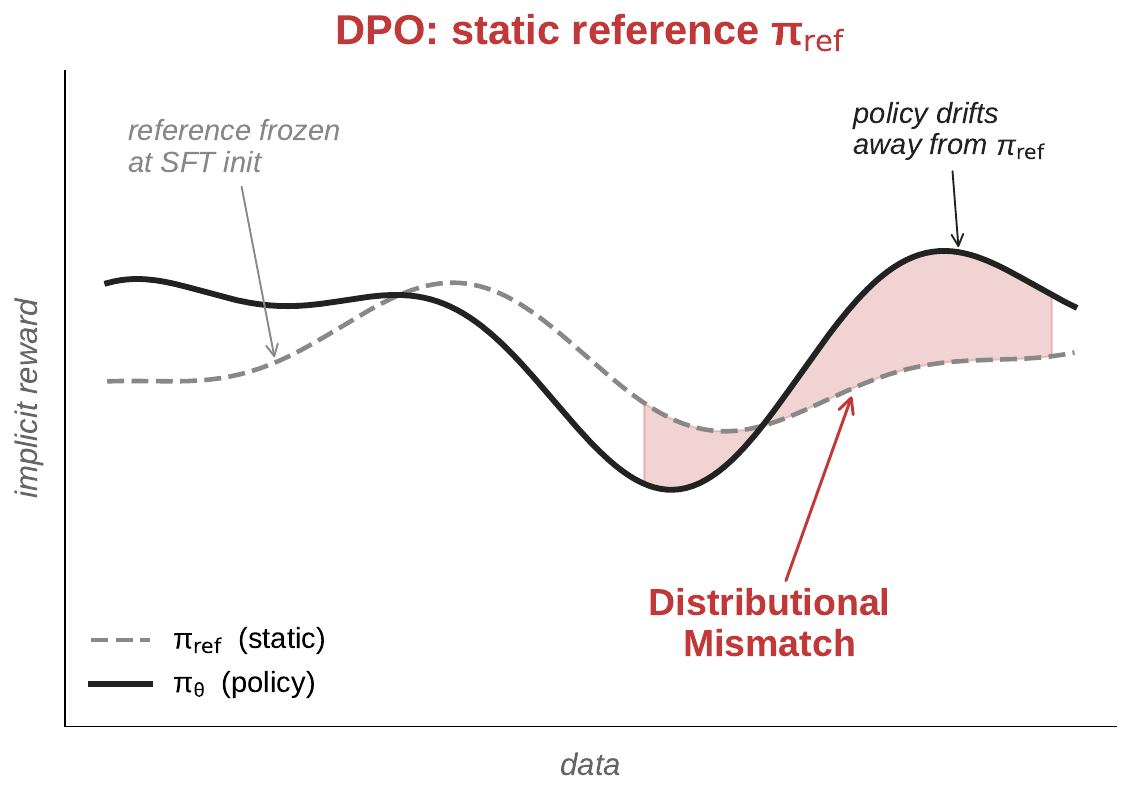}
    \caption{DPO: distributional mismatch with a static reference.}
    \label{fig:sub_1}
  \end{subfigure}
  \hfill
  \begin{subfigure}[b]{0.32\linewidth}
    \centering
    \includegraphics[width=\linewidth]{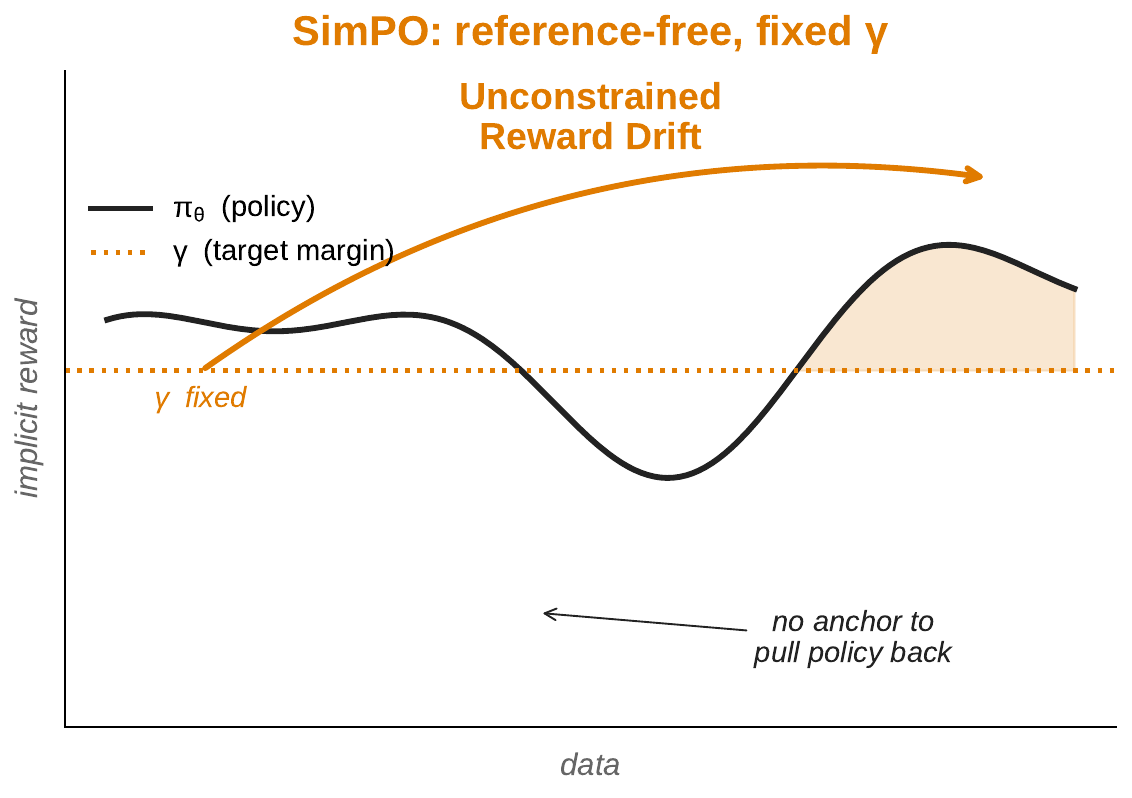}
    \caption{SimPO: unconstrained reward drift with no anchor.}
    \label{fig:sub_2}
  \end{subfigure}
  \hfill
  \begin{subfigure}[b]{0.32\linewidth}
    \centering
    \includegraphics[width=\linewidth]{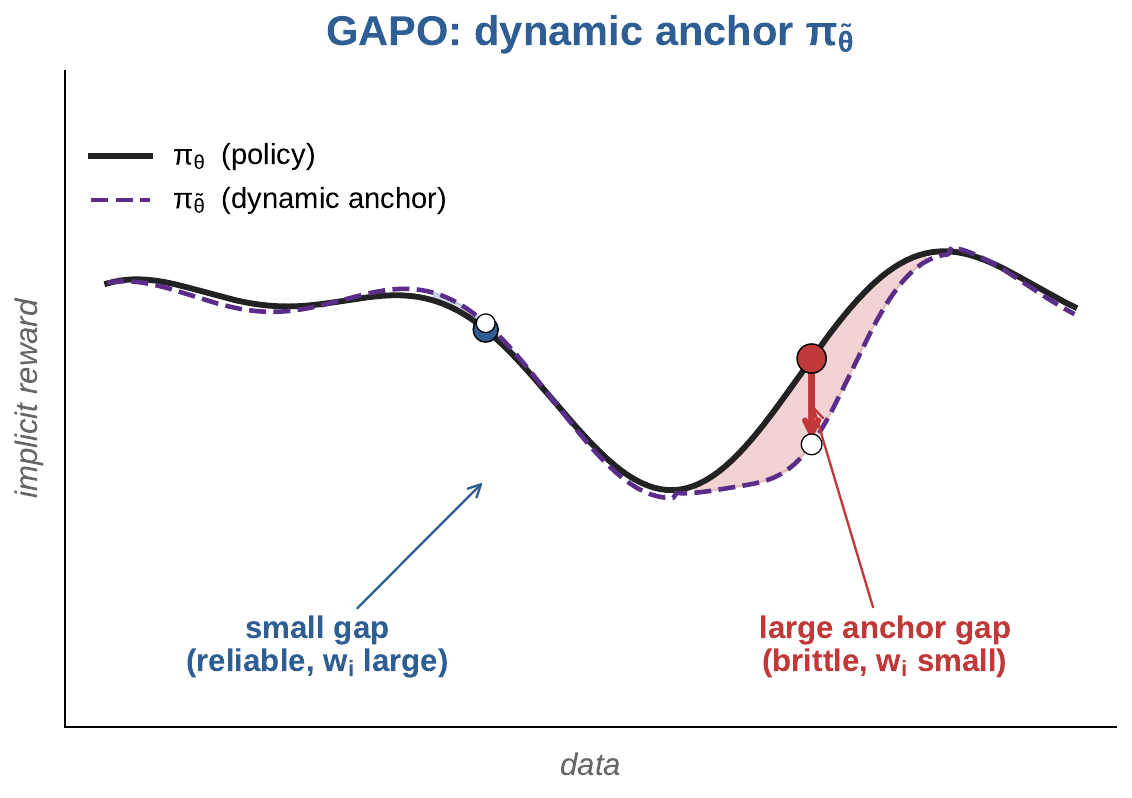}
    \caption{GAPO: dynamic anchor with adaptive Anchor Gap.}
    \label{fig:sub_3}
  \end{subfigure}

  \caption{%
    \textbf{Conceptual comparison of DPO, SimPO, and GAPO.}
    \textbf{(a)} DPO's frozen reference causes \emph{distributional
    mismatch} as $\pi_\theta$ drifts.
    \textbf{(b)} SimPO has no anchor, allowing \emph{unconstrained
    reward drift}.
    \textbf{(c)} GAPO's anchor tracks the policy and probes local
    geometry.
  }
  \label{fig:main_concept}
\end{figure*}
\section{Preliminaries}
\label{sec:preliminaries}
We consider the alignment of a policy model $\pi_\theta$ using a pairwise preference dataset $\mathcal{D} = \{(x_i, y_{w,i}, y_{l,i})\}_{i=1}^N$, where $y_{w,i}$ and $y_{l,i}$ denote the preferred and dispreferred responses to prompt $x_i$, respectively.
\paragraph{Direct Preference Optimization (DPO)}
DPO~\citep{rafailov2023direct} optimizes the policy model by minimizing a logistic (Bradley--Terry) loss over pairwise preferences. It uses a fixed reference model $\pi_{\text{ref}}$ (typically the SFT model) to regularize the deviation of the active policy. The objective is defined as:
\begin{equation}\label{eq:dpo}
    \begin{split}
    \mathcal{L}_{\text{DPO}}(\theta) = -\mathbb{E}_{(x, y_w, y_l) \sim \mathcal{D}} \Big[ & \log \sigma \Big(  \beta \log \frac{\pi_\theta(y_w|x)}{\pi_{\text{ref}}(y_w|x)}
     - \beta \log \frac{\pi_\theta(y_l|x)}{\pi_{\text{ref}}(y_l|x)} \Big) \Big],
    \end{split}
\end{equation}
where $\sigma$ is the sigmoid function and $\beta>0$ is an inverse temperature that controls the strength of reference regularization.

\paragraph{Simple Preference Optimization (SimPO)}
To address the dependency on the reference model and potential length biases, SimPO~\citep{meng2024simpo} defines a length-normalized reward $p_\theta(x, y) = \frac{1}{|y|} \log \pi_\theta(y|x)$ and introduces a target margin $\gamma$. The resulting objective eliminates $\pi_{\text{ref}}$ as follows
\begin{equation}\label{eq:simpo}
\mathcal{L}_{\text{SimPO}}(\theta)
=
-\mathbb{E}_{(x,y_w,y_l)\sim\mathcal{D}}
\left[
\log \sigma\left(
\beta \left(
p_\theta(x,y_w)-p_\theta(x,y_l)
\right)-\gamma
\right)
\right].
\end{equation}

\paragraph{Common margin-based update structure.}
DPO and SimPO differ in their reference structure: DPO defines the preference margin relative to a fixed reference model, whereas SimPO defines it directly from length-normalized policy likelihoods with a target margin.
These choices induce different trade-offs, including reference anchoring, policy flexibility, and susceptibility to reference mismatch or reward drift~\citep{pmlr-v235-xiong24a,meng2024simpo,lin2024mitigating,noukhovitch2023language}.
Despite these differences, both objectives assign per-pair update strength through a margin-based logistic loss.
GAPO builds on this common structure by preserving the underlying preference-gradient direction while introducing a local geometric diagnostic for how much each pair's margin degrades under a shared pessimistic perturbation.

%% file: tex/4_method.tex
\section{Method}
\label{sec:method}

\subsection{Batch-Level Geometric Anchor}

For each training batch, GAPO constructs one shared pessimistic
perturbation from the average preference margin of the examples in that
batch. The same perturbation is used to evaluate all pairs in the batch,
providing a common local stress test while retaining an instance-specific
response.
We adopt the length-normalized implicit reward $p_\theta(x,y)$ (as in SimPO) and define the preference margin for the $i$-th instance as
\begin{equation}\label{eq:delta-cur}
    M_i(\theta) := p_\theta(x_i,y_{w,i}) - p_\theta(x_i,y_{l,i}).
\end{equation}
For a mini-batch $\mathcal B$, let
\begin{equation}
    \bar M_{\mathcal B}(\theta)
    := \frac{1}{|\mathcal B|}\sum_{j\in\mathcal B} M_j(\theta).
\end{equation}
We construct a \emph{geometric anchor} by perturbing the current policy along a shared adversarial direction that maximally decreases the batch-average margin to first order within a local $\ell_2$ ball of radius $\rho$:
\begin{equation}\label{eq:perturbation}
    \epsilon_{\mathcal B}^{*}
    = -\rho\,
    \frac{\nabla_\theta\bar M_{\mathcal B}(\theta)}
    {\left\|\nabla_\theta\bar M_{\mathcal B}(\theta)\right\|_2},
    \qquad
    \tilde\theta = \theta + \epsilon_{\mathcal B}^{*}.
\end{equation}
We treat $\epsilon_{\mathcal B}^{*}$ as a shared pessimistic perturbation: every instance in the batch is evaluated under the same locally adversarial direction, rather than its own worst-case direction.
Using this anchor, we evaluate each instance's anchor margin,
$M_i(\tilde\theta) = p_{\tilde\theta}(x_i,y_{w,i}) - p_{\tilde\theta}(x_i,y_{l,i})$.

To ensure that the anchor acts as a frozen reference, we explicitly apply a \textit{stop-gradient} operator, denoted by $\operatorname{sg}(\cdot)$, e.g., $\operatorname{sg}(M_i(\tilde\theta))$, which returns the same value but blocks gradients from flowing through its argument.

\subsection{Objective Function and Anchor Gap}

We define the \emph{Anchor Gap} as the degradation of the preference margin under the detached anchor; a large positive gap indicates that the pair's margin would decrease substantially under the shared adversarial perturbation:
\begin{equation}\label{eq:gap}
    \Gamma_i(\theta)
    := M_i(\theta) - \operatorname{sg}(M_i(\tilde\theta)).
\end{equation}
We define the GAPO objective as a logistic loss on the Anchor Gap:
\begin{equation}\label{eq:GAPO-loss}
    \mathcal L_{\text{GAPO}}(\theta)
    = -\frac{1}{N}\sum_{i=1}^{N}
    \log\sigma\bigl(\beta\Gamma_i(\theta)-\gamma\bigr),
\end{equation}
where $\gamma>0$ is the target margin.
Note that GAPO does not optimize the anchor itself. Instead, the anchor provides a detached local baseline against which the stability of each instance margin is assessed.

\subsection{Optimization Interpretation: Adaptive Reweighting}
\label{subsec:dynamics}

Differentiating Eq.~\eqref{eq:GAPO-loss} with respect to $\theta$ yields
\begin{equation}\label{eq:gradient}
    \nabla_\theta\mathcal L_{\text{GAPO}}(\theta)
    = -\mathbb E_{(x,y_w,y_l)\sim\mathcal D}
    \left[
    \underbrace{\beta\sigma\bigl(\gamma-\beta\Gamma_i(\theta)\bigr)}_{w_i(\theta)}
    \cdot\nabla_\theta M_i(\theta)
    \right].
\end{equation}

This derivation highlights a key structural property: GAPO preserves each pair's preference-gradient direction $\nabla_\theta M_i(\theta)$ but dynamically scales its contribution by an instance-dependent weight $w_i(\theta)\in(0,\beta)$.
Pairs with large Anchor Gaps receive smaller weights, while pairs that remain stable under the same probe receive larger Anchor-Gap-based weights.

GAPO should not be interpreted as directly estimating absolute data quality or explicitly classifying examples as ``good'' or ``bad.''
Rather, $\Gamma_i(\theta)$ measures how much instance $i$'s margin degrades under the shared adversarial perturbation induced by the current mini-batch.
This geometry-aware reweighting provides a plausible mechanism for robustness under noisy supervision: in our noisy settings (Section~\ref{sec:experiments}), flipped or systematically biased pairs tend to develop larger Anchor Gaps and are therefore down-weighted more strongly.

%% file: tex/5_analysis.tex
\section{Theoretical Analysis}
\label{sec:theory}

Here we analyze the implemented batch-wise anchor. Our goal is not to establish a global convergence or robustness guarantee, but to characterize the local quantity measured by the Anchor Gap and how the shared pessimistic probe induces instance-wise reweighting.
We define the normalized batch direction
\begin{equation}
    g_{\mathcal B}(\theta)
    :=
    \nabla_\theta \bar M_{\mathcal B}(\theta),
    \qquad
    v_{\mathcal B}(\theta)
    :=
    \frac{g_{\mathcal B}(\theta)}{\|g_{\mathcal B}(\theta)\|_2}.
\end{equation}
The implemented geometric anchor is $\tilde\theta = \theta-\rho\,v_{\mathcal B}(\theta)$, equivalent to the construction in Eq.~\ref{eq:perturbation}.

\begin{proposition}[Batch-directional margin degradation]
\label{prop:batch-degradation}
Under standard smoothness assumptions and sufficiently small $\rho$,
\begin{equation}
\bar M_{\mathcal B}(\theta)-\bar M_{\mathcal B}(\tilde\theta)
=
\rho \|\nabla_\theta \bar M_{\mathcal B}(\theta)\|_2
-
\frac{\rho^2}{2}
v_{\mathcal B}(\theta)^\top
\nabla_\theta^2 \bar M_{\mathcal B}(\theta)
v_{\mathcal B}(\theta)
+
O(\rho^3).
\end{equation}
Moreover, $v_{\mathcal B}(\theta)$ maximizes the first-order decrease of the batch-average margin among all unit directions.
\end{proposition}

This proposition states two properties of the batch-wise probe: \(v_{\mathcal{B}}\) is the unit direction of steepest first-order decrease of \(\bar M_{\mathcal{B}}\), and the expansion quantifies how much the batch-average margin degrades along this probe. A spectral bound on the curvature term is given in Appendix~\ref{app:theory}. The full proof is given in Appendix~\ref{proof:thm5.1}.

\begin{proposition}[Tri-factor decomposition of the Anchor Gap]
\label{prop:tri-factor}
For each instance $i \in \mathcal B$, under the same smoothness assumptions and sufficiently small $\rho$,
\begin{equation}
\Gamma_i(\theta)
=
\rho \|\nabla_\theta M_i(\theta)\|_2 \, \alpha_{i,\mathcal B}(\theta)
-
\frac{\rho^2}{2}\, \kappa_{i,\mathcal B}(\theta)
+
O(\rho^3),
\end{equation}
where
\[
\alpha_{i,\mathcal B}(\theta)
:=
\frac{\langle \nabla_\theta M_i(\theta),\, g_{\mathcal B}(\theta)\rangle}
{\|\nabla_\theta M_i(\theta)\|_2 \, \|g_{\mathcal B}(\theta)\|_2}
\quad\text{and}\quad
\kappa_{i,\mathcal B}(\theta)
:=
v_{\mathcal B}(\theta)^\top \nabla_\theta^2 M_i(\theta)\, v_{\mathcal B}(\theta).
\]
\end{proposition}

Proposition~\ref{prop:tri-factor} clarifies the role of the Anchor Gap. The Anchor Gap is not an absolute data-quality score; it measures local margin degradation under the implemented shared pessimistic probe. A large \(\Gamma_i\) can arise from large gradient magnitude, positive alignment with the batch-average margin direction, or second-order curvature effects along \(v_{\mathcal B}\). The alignment term $\alpha_{i,\mathcal B}$ is the batch-specific factor that distinguishes GAPO from ordinary gradient-magnitude. The proof is given in Appendix~\ref{app:instance-decomp}.

%% file: tex/6_exp.tex
\section{Experiments}
\label{sec:experiments}
We evaluate GAPO across four open-weight backbones on instruction-following, reasoning, and noise-robustness benchmarks; full implementation details, hyperparameters, and dataset/baseline lists are in Appendix~\ref{app:implementation}. We organize our analysis around three questions motivated by the local-stability view: (1) Does conservative attenuation of locally brittle pairs improve alignment without sacrificing reasoning? (2) Does it improve robustness when corrupted or spurious preferences are present? (3) Does the batch-conditioned Anchor Gap yield a meaningful instance-wise reweighting signal beyond gradient magnitude alone?

\subsection{Main Alignment and Reasoning Results}

\input{tab/main/1_align}
\paragraph{Alignment Performance.}
Table~\ref{tab:main_results} summarizes instruction-following performance across model families.
Overall, GAPO consistently matches or improves alignment metrics relative to strong reference-based and reference-free baselines.
The gains are not confined to a single benchmark: we observe improvements on both AlpacaEval~2.0 and Arena-Hard, suggesting that the benefit is not merely an artifact of a particular evaluator. The improvement is also consistent across three additional training seeds on Mistral-7B: GAPO achieves $34.51\pm0.72$ AlpacaEval~2 LC
compared with $31.44\pm1.37$ for SimPO; full per-seed results are
reported in Appendix~C.2.

\input{tab/main/2_reasoning}

\paragraph{Reasoning Preservation.}
A common pitfall of preference optimization is the trade-off between improved instruction-following and degraded reasoning (the \emph{alignment tax}).
Table~\ref{tab:reasoning} shows that GAPO improves alignment while largely preserving reasoning performance.
In particular, compared to reference-free training that can exhibit drift, GAPO tends to maintain stronger reasoning scores while still delivering alignment gains.

\subsection{Robustness Under Noisy Supervision}
Noisy labels can create preference pairs whose margins are locally brittle along the current batch direction. We evaluate GAPO under random label flips and a structured length-dependent noise setting that selectively flips pairs based on response length (creating a ``shorter is better'' bias). Figure~\ref{fig:noise_robustness} reports Reward Accuracy on the test set.
\begin{figure}[!htb]
    \centering
    \includegraphics[width=0.8\linewidth]{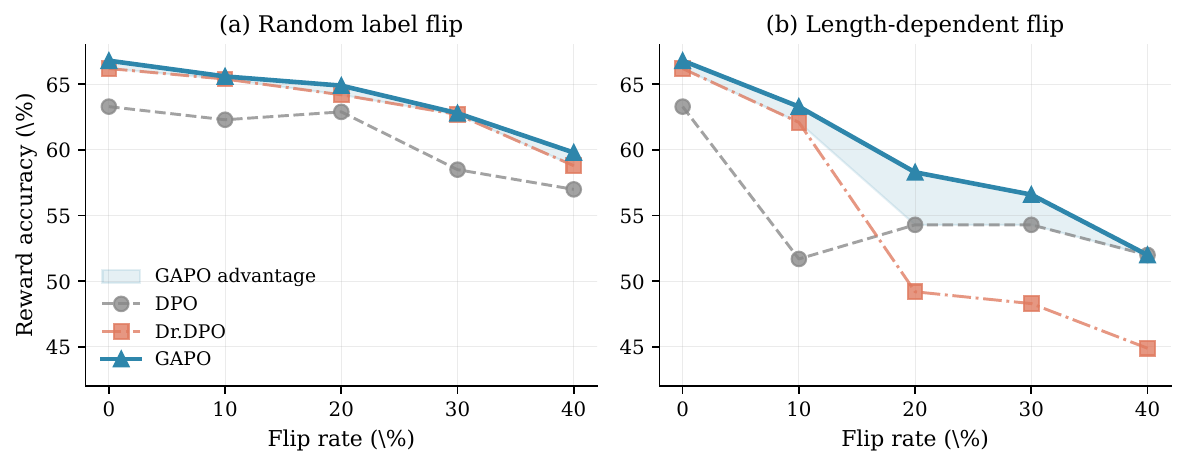}
    \caption{\textbf{Robustness against label noise.} Reward accuracy under (a) random label flips and (b) length-dependent label flips, evaluated on Pythia-2.8B trained on Anthropic HH. The shaded region in each panel marks the gap between GAPO and the strongest baseline at each flip rate. Exact numerical values are reported in Appendix~\ref{app:noise_table}.}
    \label{fig:noise_robustness}
\end{figure}

Figure~\ref{fig:noise_robustness} shows that GAPO improves over dedicated noise-robust baselines across noise rates and noise types, without explicitly modeling noise. This is consistent with the Anchor-Gap-based interpretation: GAPO attenuates large-gap pairs, which are enriched in noisy supervision in our diagnostic settings.
The structured length-dependent setting reveals a sharper distinction: at flip rates of $20\%$ and above, the explicit-noise-modeling baseline (Dr.DPO) collapses below DPO, whereas GAPO degrades smoothly and remains the strongest method, consistent with conservative attenuation generalizing beyond random label corruption.
Appendix~\ref{app:prospective} provides additional analysis of the induced per-instance ranking as an auxiliary signal, not as a data-quality estimate.

\begingroup
\setlength{\intextsep}{0pt}
\begin{wraptable}{r}{0.44\textwidth}
    \centering
    \caption{%
    \textbf{Robustness in the main Mistral-7B alignment setting.}
    Mistral-7B-Instruct trained on UltraFeedback with $20\%$ random
    preference-label flips. AlpacaEval~2 length-controlled (LC) and
    raw win rates (WR) are reported in percent.
    }
    \label{tab:noise_7b}
    \begin{tabular}{lcc}
        \toprule
        Method & AE2 LC (\%) & AE2 WR (\%) \\
        \midrule
        SimPO  & 25.1 & 24.6 \\
        Dr.DPO & 27.4 & 23.3 \\
        GAPO   & \textbf{27.8} & \textbf{29.9} \\
        \bottomrule
    \end{tabular}
\end{wraptable}
\paragraph{Robustness in the main Mistral-7B setting.}
We additionally train Mistral-7B-Instruct on UltraFeedback with
$20\%$ random preference-label flips, using the same model--dataset
setting as our main alignment experiments.
Table~\ref{tab:noise_7b} shows that GAPO improves over SimPO by
$2.72$ AlpacaEval~2 LC points ($27.83$ vs.\ $25.11$) and achieves
slightly higher LC than Dr.DPO ($27.44$), while also attaining the
highest raw win rate among the compared methods.
These results extend the downstream robustness evaluation from
Pythia-2.8B/Anthropic HH to the main Mistral-7B/UltraFeedback setting.

To examine a possible mechanism for robustness under noisy supervision,
we next investigate whether corrupted preference pairs receive smaller
update weights. Figure~\ref{fig:mechanism} tracks the weights assigned
to clean and flipped pairs during training and the flip rates across
weight tiers.
\par
\endgroup

\subsection{Diagnosing the Anchor-Gap Reweighting Signal}

\label{sec:mechanism-empirical}

\subsubsection{GAPO down-weights corrupted supervision}
\label{sec:retrospective}

We examine \emph{whether} the Anchor Gap $\Gamma_i$ tracks per-instance margin degradation along the current batch direction during training. We use flipped labels as a controlled probe of corrupted supervision and test whether they become enriched among strongly attenuated pairs: under random-flip preference noise ($20\%$), we track the GAPO weight $w_i = \beta\,\sigma(\gamma - \beta\Gamma_i)$ assigned to each pair across training.

\paragraph{GAPO assigns progressively smaller weights to flipped preferences (Fig.~\ref{fig:mechanism}a).}
Using the GAPO weight $w_i = \beta\,\sigma(\gamma - \beta\Gamma_i)$, the mean weight assigned to flipped preferences becomes progressively smaller than that of clean preferences over training. By the end of the first epoch, flipped pairs receive a $16.9\%$ smaller weight on average, consistent with the conservative-attenuation interpretation: brittle pairs along the batch direction develop larger Anchor Gaps and correspondingly smaller weights.

\paragraph{Low-weight tier is enriched in flipped supervision.}
Stratifying examples directly by the GAPO weight $w_i$ produces a sharper divergence (Fig.~\ref{fig:mechanism}b). The bottom $20\%$ weight tier (most strongly down-weighted by GAPO) accumulates a flip rate of $30.1\%$ by the final checkpoint, while the top $20\%$ tier (least down-weighted) drops to $16.9\%$, a $13.2\%$ gap. The pairs GAPO most strongly down-weights are statistically enriched in corrupted supervision in this controlled setting.

\begin{figure*}[!htb]
  \centering
  \begin{subfigure}[b]{0.4\linewidth}
    \centering
    \includegraphics[width=\linewidth]{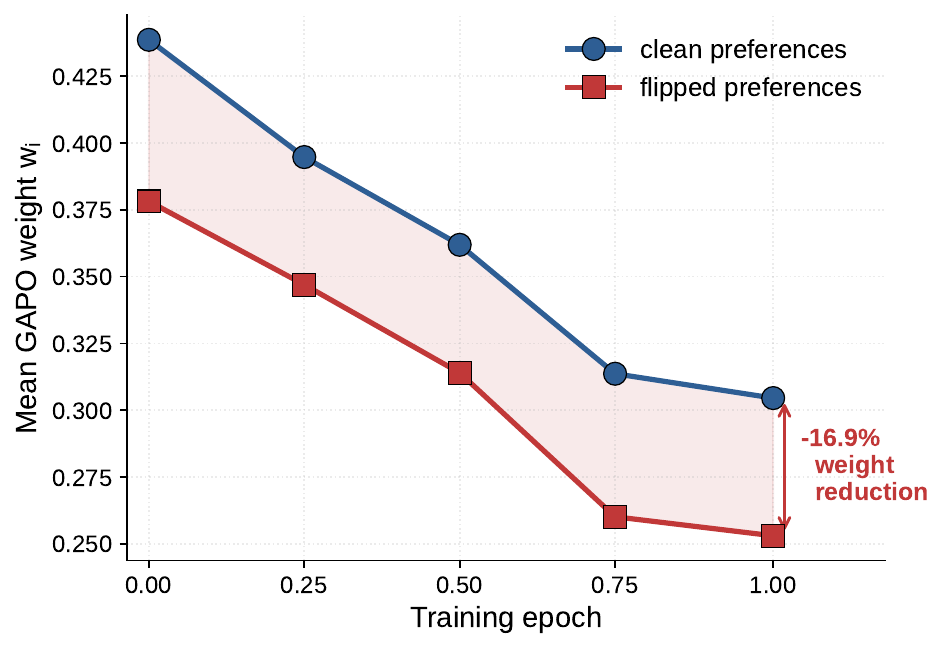}
    \caption{GAPO weight: clean vs.\ flipped.}
    \label{fig:mechanism-a}
  \end{subfigure}
  \begin{subfigure}[b]{0.4\linewidth}
    \centering
    \includegraphics[width=\linewidth]{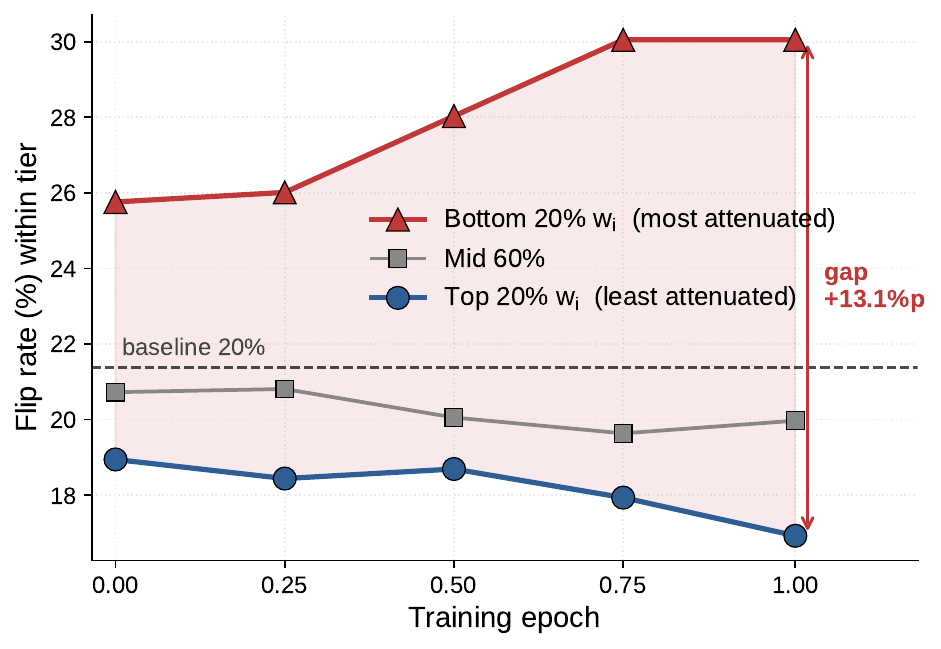}
    \caption{Flip rate by GAPO weight tier.}
    \label{fig:mechanism-b}
  \end{subfigure}
  \caption{%
    \textbf{GAPO's Anchor-Gap reweighting is statistically biased toward
    flipped supervision during training}
    \textbf{(a)} Mean GAPO weight $w_i = \beta\,\sigma(\gamma - \beta\Gamma_i)$
    separates flipped from clean preferences over training.
    \textbf{(b)} Stratifying examples directly by $w_i$ at each checkpoint.
    Detailed numerical results are reported in Appendix~\ref{app:mechanism-stats}.
  }
  \label{fig:mechanism}
\end{figure*}
\subsubsection{The batch-wise probe remains informative at the instance level}
\label{sec:batch-wise-validation}

The preceding results show that GAPO's weights are biased toward corrupted supervision.
A remaining concern is whether per-instance GAPO weights remain meaningful when computed from a shared batch-wise perturbation rather than instance-specific worst-case probes.
This diagnostic tests a limited but important property: whether the shared batch direction preserves a useful first-order ranking of locally sensitive instances.
Proposition~\ref{prop:tri-factor} shows that the first-order term of the batch-wise Anchor Gap equals the instance gradient magnitude scaled by the batch-alignment coefficient \(\alpha_{i,\mathcal B}\).
We therefore test whether the batch-wise probe induces a per-instance ranking comparable to its instance-wise first-order counterpart.

We compute, for each preference pair, a batch-directional first-order proxy
$\widehat{\Gamma}_i^{\text{batch}}
= \rho \|\nabla_\theta M_i\|_2 \alpha_{i,\mathcal B}$, whereas the corresponding instance-directional proxy is $\widehat{\Gamma}_i^{\text{inst}}
= \rho \|\nabla_\theta M_i\|_2$.

The batch-wise anchor moves in the same margin-decreasing half-space as each instance's own first-order worst-case direction in \(99.4\%\) of pairs (Figure~\ref{fig:batch-validation-a}), indicating that the shared perturbation is rarely adversarially misaligned with individual instances.
Although the shared direction is not identical to each instance's worst-case direction, the resulting ranking of preference pairs closely matches the instance-directional proxy (Figure~\ref{fig:batch-validation-b}).
The match is strongest in the most relevant tier: the top-\(20\%\) most-downweighted subsets of the two rankings overlap by \(86.1\%\) (Figure~\ref{fig:batch-validation-c}).
This suggests that the shared anchor is a practical ranking proxy for deciding which pairs to attenuate, while avoiding the cost of per-instance perturbations.

We emphasize that this is a first-order ranking validation rather than a claim of equivalence between batch-wise and instance-wise worst-case degradation.
It also does not imply that gradient magnitude alone explains GAPO's gains; the norm-only alternative is evaluated separately in next section.

\begin{figure*}[!htb]
  \centering
  \begin{subfigure}[t]{0.32\linewidth}
    \includegraphics[width=\linewidth]{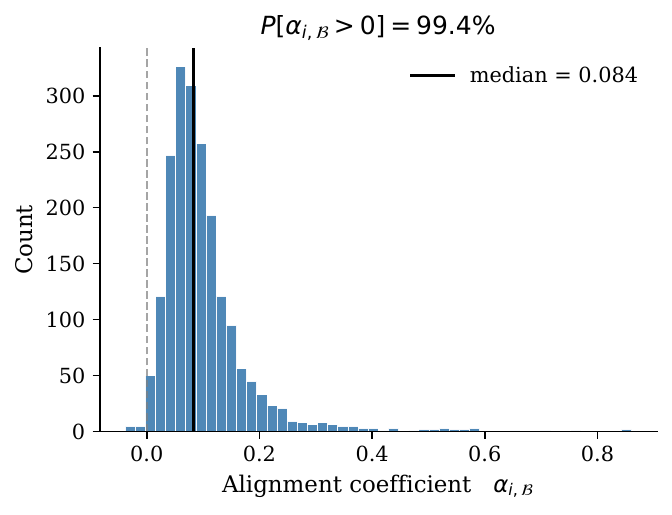}
     \subcaption{Distribution of $\alpha_{i,\mathcal{B}}$.}
         \label{fig:batch-validation-a}
  \end{subfigure}\hfill
  \begin{subfigure}[t]{0.32\linewidth}
    \includegraphics[width=\linewidth]{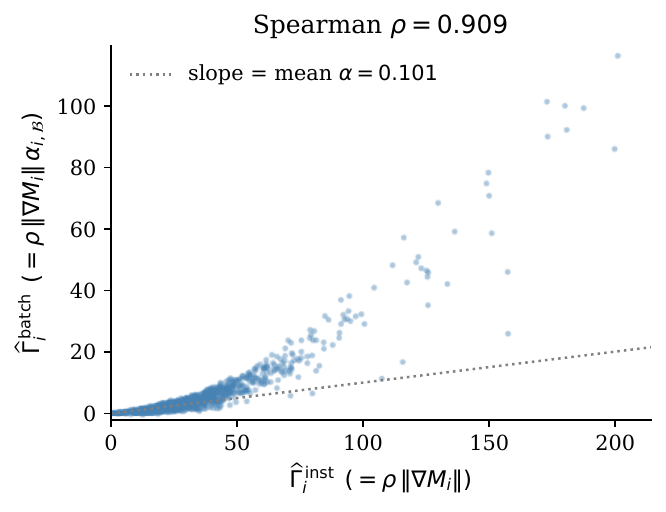}
    \subcaption{Ranking agreement between the batch-directional first-order proxy.}
    \label{fig:batch-validation-b}
  \end{subfigure}\hfill
  \begin{subfigure}[t]{0.32\linewidth}
    \includegraphics[width=\linewidth]{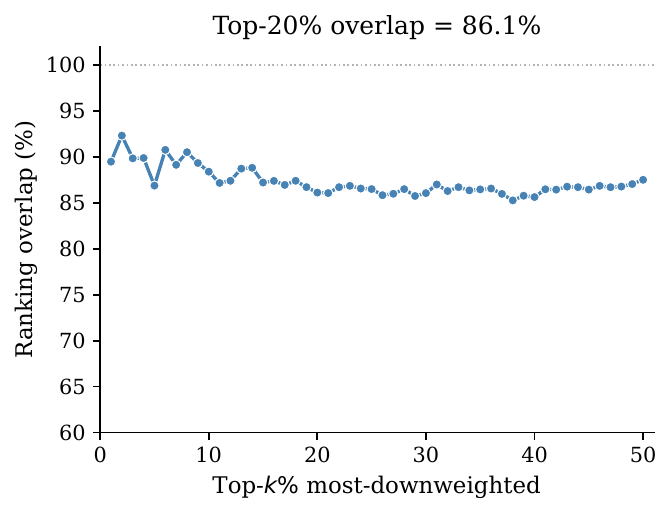}
     \subcaption{Top-$k\%$ overlap of the most-downweighted samples
    between the two rankings.}
\label{fig:batch-validation-c}
  \end{subfigure}
\caption{
\textbf{Batch-wise probe validation.}
\textbf{(a)} Distribution of the alignment coefficient \(\alpha_{i,\mathcal B}\).
\textbf{(b)} Ranking agreement between the batch-directional first-order proxy
 and the instance-directional proxy.
\textbf{(c)} Top-\(k\%\) overlap of the most-downweighted samples between the two rankings.
Detailed analysis is in Appendix~\ref{app:anchor-validation}.
}
  \label{fig:batch-validation}
\end{figure*}

\subsubsection{Gradient magnitude alone is insufficient}
\label{sec:gradnorm-ablation}

To test whether GAPO's gains are explained by gradient magnitude alone, we compare against GradNorm-Reweight, which replaces $\Gamma_i$ with a normalized per-instance margin-gradient norm while keeping the same reweighting form and training setup.

\begin{table}[!htb]
\centering
\caption{Gradient-norm reweighting ablation on Mistral-Instruct.}
\label{tab:gradnorm_ablation}
\resizebox{0.75\linewidth}{!}{
\begin{tabular}{lcccc}
\toprule
Method & AlpacaEval 2 LC & AlpacaEval 2 WR & GSM8k & ARC-C \\
\midrule
SimPO & 32.1 & 34.8 & 36.6 & 66.3 \\
GradNorm-Reweight & 33.5 & 34.7 & 37.0 & 66.1 \\
GAPO & \textbf{35.7} & \textbf{38.7} & \textbf{40.3} & \textbf{66.3} \\
\bottomrule
\end{tabular}
}
\end{table}

Table~\ref{tab:gradnorm_ablation} shows that GradNorm-Reweight improves over SimPO on some metrics but remains below GAPO, suggesting that norm-only attenuation is insufficient.
This is consistent with Proposition~\ref{prop:tri-factor}: the Anchor Gap depends not only on gradient magnitude, but also on batch-direction alignment and the directional curvature that together determine how much each margin degrades under the shared pessimistic probe.

\subsection{Additional Analyses}
\label{subsec:ablations-summary}

Appendix~\ref{app:ablations} provides ablations on perturbation strategy, perturbation radius, anchor batch multiplier, perturbation scope, and SimPO+SAM comparisons.
These analyses support the use of a shared batch-wise probe and help distinguish Anchor-Gap reweighting from uniform optimizer-level flatness-seeking.
In particular, batch-wise perturbations provide a stronger overall trade-off than instance-wise variants, $\rho=0.05$ gives the best perturbation radius in our setting, an $8\times$ anchor batch multiplier balances directional stability and local sensitivity, and full-model perturbation outperforms LM-head-only probing.
As an auxiliary prospective check, Appendix~\ref{app:prospective} further shows that low-$\Gamma$ subsets selected at the SFT checkpoint yield a better accuracy--drift trade-off for a separate SimPO learner, suggesting that the Anchor Gap carries useful stability information beyond GAPO's online reweighting.
Appendix~\ref{app:hessian} additionally reports a targeted LM-head curvature probe, which we treat as auxiliary evidence rather than the main explanation.
\paragraph{Computational cost.}
GAPO requires an additional anchor-gradient computation and anchor forward pass, resulting in roughly a \(2\times\) per-step wall-clock overhead relative to SimPO.
To test whether the gains are explained by extra training time, Figure \ref{fig:efficiency_stability_app} reports a wall-clock-matched comparison on Mistral-7B: two epochs of SimPO reach 33.0 AlpacaEval~2 LC, whereas one epoch of GAPO reaches 35.7. Thus, GAPO's improvement is not explained solely by longer SimPO training under the same recipe.
Figure~\ref{fig:efficiency_stability_app} also shows that GAPO reaches 33.0 AlpacaEval~2 LC after 0.75 epoch, exceeding SimPO's 32.1 after one epoch with 25\% fewer training examples processed, suggesting improved training-data efficiency in this setting.

\begin{figure}[!htb]
    \centering
    \includegraphics[width=0.65\columnwidth]{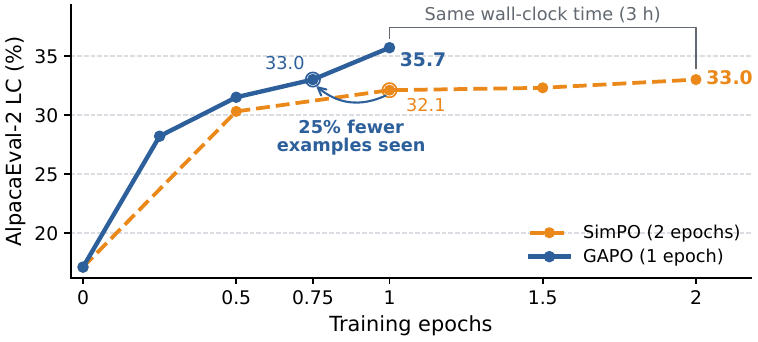}
    \caption{
    Wall-clock-matched diagnostic on Mistral-7B  under the same hardware and training recipe.
    }
    \label{fig:efficiency_stability_app}
\end{figure}

%% file: tab/main/1_align.tex
\begin{table*}[!htb]
    \centering
    \caption{\textbf{Main Results on Instruction-Following Benchmarks.} We compare GAPO with various baselines across four models. GAPO is competitive with or improves over strong baselines across most model-benchmark settings. \textbf{LC}: Length-Controlled, \textbf{WR}: Win Rate.}
    \label{tab:main_results}
    \setlength{\tabcolsep}{1pt}
    \resizebox{\textwidth}{!}{%
    \begin{tabular}{lccc@{\hspace{6pt}}ccc@{\hspace{6pt}}ccc@{\hspace{6pt}}ccc}
        \toprule
        \multirow{4}{*}{\textbf{Method}}
        & \multicolumn{3}{c}{\shortstack{\textbf{Mistral-Instruct}\\\textbf{(7B)}}}
        & \multicolumn{3}{c}{\shortstack{\textbf{Llama-3-Base}\\\textbf{(8B)}}}
        & \multicolumn{3}{c}{\shortstack{\textbf{Llama-3-Instruct}\\\textbf{(8B)}}}
        & \multicolumn{3}{c}{\shortstack{\textbf{Gemma-2-Instruct}\\\textbf{(9B)}}} \\
        \cmidrule(lr){2-4}\cmidrule(lr){5-7}\cmidrule(lr){8-10}\cmidrule(lr){11-13}
        & \multicolumn{2}{c}{\textbf{AlpacaEval 2}} & \shortstack{\textbf{Arena-}\\\textbf{Hard}}
        & \multicolumn{2}{c}{\textbf{AlpacaEval 2}} & \shortstack{\textbf{Arena-}\\\textbf{Hard}}
        & \multicolumn{2}{c}{\textbf{AlpacaEval 2}} & \shortstack{\textbf{Arena-}\\\textbf{Hard}}
        & \multicolumn{2}{c}{\textbf{AlpacaEval 2}} & \shortstack{\textbf{Arena-}\\\textbf{Hard}} \\
        \cmidrule(lr){2-3}\cmidrule(lr){4-4}\cmidrule(lr){5-6}\cmidrule(lr){7-7}
        \cmidrule(lr){8-9}\cmidrule(lr){10-10}\cmidrule(lr){11-12}\cmidrule(lr){13-13}
        & \textbf{LC (\%)} & \textbf{WR (\%)} & \textbf{WR (\%)}
        & \textbf{LC (\%)} & \textbf{WR (\%)} & \textbf{WR (\%)}
        & \textbf{LC (\%)} & \textbf{WR (\%)} & \textbf{WR (\%)}
        & \textbf{LC (\%)} & \textbf{WR (\%)} & \textbf{WR (\%)} \\
        \midrule
        SFT & 17.1 & 14.7 & 12.6 & 6.2 & 4.6 & 3.3 & 26.0 & 25.3 & 22.3 & 48.7 & 36.5 & 42.1 \\
        \midrule
        DPO & 26.8 & 24.9 & 16.3 & 18.2 & 15.5 & 15.9 & 40.3 & 37.9 & 32.6 & 70.4 & 66.9 & 58.8 \\
        IPO & 20.3 & 20.3 & 16.2 & 14.4 & 14.2 & 17.8 & 35.6 & 35.6 & 30.5 & 62.6 & 58.4 & 53.5 \\
        CPO & 23.8 & 28.8 & 22.6 & 10.8 & 8.1 & 5.8 & 28.9 & 32.2 & 28.8 & 56.4 & 53.4 & 55.2 \\
        KTO & 24.5 & 23.6 & 17.9 & 14.2 & 12.4 & 12.5 & 33.1 & 31.8 & 26.4 & 61.7 & 55.5 & 53.8 \\
        ORPO & 24.5 & 24.9 & 20.8 & 12.2 & 10.6 & 10.8 & 28.5 & 27.4 & 25.8 & 56.2 & 46.7 & 46.2 \\
        R-DPO & 27.3 & 24.5 & 16.1 & 17.6 & 14.4 & 17.2 & 41.1 & 37.8 & 33.1 & 68.3 & 66.9 & 57.9 \\
        SimPO & 32.1 & 34.8 & 21.0 & 22.0 & 20.3 & 23.4 & 44.7 & 40.5 & \textbf{33.8} & 72.4 & 65.0 & 57.8 \\
        $\alpha$-DPO & 32.3 & 32.6 & 21.7 & 18.3 & 13.8 & 22.5 & 46.6 & 38.1 & 33.3 & 73.4 & 66.1 & 60.8 \\
        \midrule
        \textbf{GAPO} & \textbf{35.7} & \textbf{38.7} & \textbf{22.7} & \textbf{25.0} & \textbf{23.7} & \textbf{30.3} & \textbf{47.4} & \textbf{42.8} & 33.7 & \textbf{74.0} & \textbf{67.7} & \textbf{61.5} \\
        \bottomrule
    \end{tabular}%
    }
\end{table*}

%% file: tab/main/2_reasoning.tex
\begin{table}[!htb]
    \centering
    \caption{\textbf{Reasoning vs. Alignment Trade-off across Models} Performance comparison on Mistral-Instruct and Llama-3-Base.}
    \label{tab:reasoning}
    \resizebox{0.7\columnwidth}{!}{%
        \begin{tabular}{lccc|ccc}
        \toprule
        \textbf{Model} & \multicolumn{3}{c}{\textbf{Mistral-7B}} & \multicolumn{3}{c}{\textbf{Llama-3-8B}} \\
        \cmidrule(lr){2-4}\cmidrule(lr){5-7}
        \textbf{Method} & \textbf{AlpacaEval 2} & \textbf{GSM8k} & \textbf{ARC-C}
        & \textbf{AlpacaEval 2} & \textbf{GSM8k} & \textbf{ARC-C} \\
        \midrule
        DPO & 26.8 & 41.8 & 66.7 & 18.2 & 55.5 & 65.6 \\
        RDPO & 27.3 & \textbf{42.6} & \textbf{67.5} & 17.6 & 54.9 & 65.2 \\
        SimPO & 32.1 & 36.6 & 66.3 & 22.0 & 53.0 & 65.2 \\
        $\alpha$-DPO & 32.3 & 29.3 & 65.2 & 18.3 & 45.4 & 66.0 \\
        \textbf{GAPO} & \textbf{35.7} & 40.3 & 66.3 & \textbf{25.0} & \textbf{56.0} & \textbf{67.2} \\
        \bottomrule
        \end{tabular}%
    }
\end{table}

%% file: tex/7_con.tex
\section{Conclusion} \label{sec:conclusion}
We presented \textbf{Geometric Anchor Preference Optimization (GAPO)}, a geometry-aware framework for offline preference optimization. GAPO constructs a pessimistic local anchor around the current policy and uses the resulting \emph{Anchor Gap} to reweight preference pairs according to their local margin degradation along the current batch direction. Our analysis characterizes the Anchor Gap as a batch-directional proxy for local instability, decomposing it into gradient magnitude, batch-direction alignment, and directional curvature. Empirically, GAPO improves alignment across model families while largely preserving reasoning ability, and remains robust under random and structured preference noise. Overall, these results suggest that local geometric degradation provides a useful signal for stable preference optimization beyond fixed-reference anchoring or explicit noise modeling.

%% file: tex/funding_disclosure.tex
This work was supported by PYLER, which also provided all computing resources.
No third-party funding was received for this work, and the authors declare no competing interests.

%% file: tex/appendix.tex
\appendix
\setcounter{table}{0}
\renewcommand{\thetable}{\Alph{table}}

\setcounter{figure}{0}
\renewcommand{\thefigure}{\Alph{figure}}

\section*{Appendix Overview}
\label{app:overview}

\noindent
The appendix is organized as follows.

\begin{itemize}[leftmargin=2em,itemsep=3pt,topsep=4pt]
  \item \hyperref[app:theory]{\textbf{Appendix A: Theoretical Derivations.}}
    Gradient update derivation, proof of Proposition~\ref{prop:batch-degradation}, and proof of Proposition~\ref{prop:tri-factor} (instance-wise Anchor-Gap decomposition).
  \item \hyperref[app:implementation]{\textbf{Appendix B: Implementation Details.}}
    Training hyperparameters, GAPO pseudocode, and optimization and evaluation setup.
  \item \hyperref[app:main-numerics]{\textbf{Appendix C: Main Experimental Results.}}
    Full numerical results for the noise-robustness experiments in the main text.
    Additional training-seed results are reported in Section~\ref{app:seed-evaluation}.
  \item \hyperref[app:mechanism-stats]{\textbf{Appendix D: Mechanistic Statistics.}}
    Detailed weight-separation statistics and weight-tier flip rates underlying Figure~\ref{fig:mechanism}.
  \item \hyperref[app:prospective]{\textbf{Appendix E: Prospective Anchor-Gap Validation.}}
    Subset selection protocol, fraction sweep, and training dynamics under anchor-gap selection.
  \item \hyperref[app:anchor-validation]{\textbf{Appendix F: Batch-Wise Anchor Validation.}}
    Ranking-consistency probe compared against instance-wise validation.
  \item \hyperref[app:ablations]{\textbf{Appendix G: Ablations and Comparisons.}}
    Strategy, radius, anchor batch multiplier, and scope ablations; comparison with SimPO+SAM.
  \item \hyperref[app:efficiency]{\textbf{Appendix H: Computational Efficiency.}}
    Wall-clock training dynamics and runtime analysis.
    Reference-based baselines are included in Table~\ref{tab:compute-tradeoff}.
  \item \hyperref[app:hessian]{\textbf{Appendix I: Hessian Eigenspectrum Analysis.}}
    LM-head curvature probe and supporting eigenspectrum results.
\end{itemize}

\section{Theoretical Derivations}
\label{app:theory}

In this section, we provide detailed derivations of the GAPO gradient update and prove Proposition~\ref{prop:batch-degradation}, which characterizes the implemented batch-wise anchor as a batch-directional probe of local margin degradation.

\subsection{Derivation of GAPO Gradient Update Rule}
Recall the GAPO objective function defined in Eq.\ref{eq:GAPO-loss} :
\begin{equation}
    \mathcal{L}(\theta) = -\frac{1}{N} \sum_{i=1}^{N} \log \sigma \left( \beta \Gamma_i(\theta) - \gamma \right),
\end{equation}
where the Anchor Gap is $\Gamma_i(\theta) = M_i(\theta) - \text{sg}(M_i(\tilde{\theta}))$. Note that the anchor term $M_i(\tilde{\theta})$ is treated as a constant (stop-gradient) during the update step.

Let $z_i = \beta \Gamma_i(\theta) - \gamma$. The gradient with respect to $\theta$ is:
\begin{align}
    \nabla_\theta \mathcal{L}(\theta) &= -\frac{1}{N} \sum_{i=1}^{N} \frac{\partial}{\partial \theta} \log \sigma(z_i) \\
    &= -\frac{1}{N} \sum_{i=1}^{N} \left( 1 - \sigma(z_i) \right) \frac{\partial z_i}{\partial \theta} \\
    &= -\frac{1}{N} \sum_{i=1}^{N} \sigma(-z_i) \cdot \beta \nabla_\theta \Gamma_i(\theta).
\end{align}
Since $\nabla_\theta \Gamma_i(\theta) = \nabla_\theta M_i(\theta)$, we substitute $z_i$:
\begin{equation}
    \nabla_\theta \mathcal{L}(\theta) = -\frac{1}{N} \sum_{i=1}^{N} \left[ \beta \sigma \left( -(\beta \Gamma_i(\theta) - \gamma) \right) \right] \nabla_\theta M_i(\theta).
\end{equation}
By defining the instance-dependent weight as $w_i(\theta) = \beta \sigma(\gamma - \beta \Gamma_i(\theta))$, we recover Eq. \ref{eq:gradient} from the main text:
\begin{equation}
    \nabla_\theta \mathcal{L}(\theta) = -\mathbb{E}_{\mathcal{D}} \left[ w_i(\theta) \cdot \nabla_\theta M_i(\theta) \right].
\end{equation}
This confirms that GAPO performs weighted margin maximization, where the weight $w_i$ decays as the instability gap $\Gamma_i$ increases.

\subsection{Proof of Proposition~\ref{prop:batch-degradation}}
\label{proof:thm5.1}
\paragraph{Proof of Proposition~\ref{prop:batch-degradation}.}
Define the one-dimensional function
\[
f(t):=\bar M_{\mathcal B}\bigl(\theta-t\rho v_{\mathcal B}(\theta)\bigr),
\qquad t\in[0,1].
\]
Since \(\bar M_{\mathcal B}\) is three times continuously differentiable on a neighborhood of \(\mathbb B(\theta,\rho)\), the third-order Taylor expansion of \(f\) around \(t=0\) yields
\[
f(1)
=
f(0)+f'(0)+\frac{1}{2}f''(0)+R_{\mathcal B}(\theta,\rho),
\]
where the remainder satisfies
\[
|R_{\mathcal B}(\theta,\rho)|
\le
\frac{L_{\mathcal B}}{6}\rho^3.
\]

We now compute the derivatives of \(f\). First,
\[
f(0)=\bar M_{\mathcal B}(\theta).
\]
By the chain rule,
\[
f'(t)
=
-\rho\,
\Big\langle
\nabla_\theta \bar M_{\mathcal B}\bigl(\theta-t\rho v_{\mathcal B}(\theta)\bigr),
\,v_{\mathcal B}(\theta)
\Big\rangle,
\]
and therefore
\[
f'(0)
=
-\rho\,
\langle g_{\mathcal B}(\theta),v_{\mathcal B}(\theta)\rangle
=
-\rho\,\|g_{\mathcal B}(\theta)\|_2,
\]
since \(v_{\mathcal B}(\theta)=g_{\mathcal B}(\theta)/\|g_{\mathcal B}(\theta)\|_2\).

Differentiating once more,
\[
f''(t)
=
\rho^2\,
v_{\mathcal B}(\theta)^\top
\nabla_\theta^2 \bar M_{\mathcal B}\bigl(\theta-t\rho v_{\mathcal B}(\theta)\bigr)
v_{\mathcal B}(\theta),
\]
so
\[
f''(0)
=
\rho^2\,
v_{\mathcal B}(\theta)^\top
\nabla_\theta^2 \bar M_{\mathcal B}(\theta)
v_{\mathcal B}(\theta).
\]

Substituting these into the Taylor expansion gives
\[
\bar M_{\mathcal B}(\tilde\theta)
=
\bar M_{\mathcal B}(\theta)
-\rho \|g_{\mathcal B}(\theta)\|_2
+\frac{\rho^2}{2}
v_{\mathcal B}(\theta)^\top
\nabla_\theta^2 \bar M_{\mathcal B}(\theta)
v_{\mathcal B}(\theta)
+R_{\mathcal B}(\theta,\rho).
\]
Rearranging yields
\[
\bar M_{\mathcal B}(\theta)-\bar M_{\mathcal B}(\tilde\theta)
=
\rho \|g_{\mathcal B}(\theta)\|_2
-
\frac{\rho^2}{2}
v_{\mathcal B}(\theta)^\top
\nabla_\theta^2 \bar M_{\mathcal B}(\theta)
v_{\mathcal B}(\theta)
+R_{\mathcal B}(\theta,\rho),
\]
with \(|R_{\mathcal B}(\theta,\rho)|\le \frac{L_{\mathcal B}}{6}\rho^3\), which proves the claimed expansion.

Finally, let \(u\) be any unit vector. The first-order decrease of the batch-average margin along the perturbation \(\theta-\rho u\) is
\[
\rho \langle g_{\mathcal B}(\theta),u\rangle.
\]
By the Cauchy--Schwarz inequality,
\[
\langle g_{\mathcal B}(\theta),u\rangle
\le
\|g_{\mathcal B}(\theta)\|_2\|u\|_2
=
\|g_{\mathcal B}(\theta)\|_2,
\]
with equality if and only if \(u\) is parallel to \(g_{\mathcal B}(\theta)\), i.e., \(u=v_{\mathcal B}(\theta)\). Hence \(v_{\mathcal B}(\theta)\) maximizes the first-order decrease of the batch-average margin among all unit directions.
\hfill\(\square\)

\begin{corollary}[Spectral lower bound]
\label{cor:spectral}
Under the assumptions of Proposition~\ref{prop:batch-degradation},
\begin{equation}
\bar M_{\mathcal B}(\theta)-\bar M_{\mathcal B}(\tilde\theta)
\ge
\rho \|\nabla_\theta \bar M_{\mathcal B}(\theta)\|_2
-
\frac{\rho^2}{2}
\lambda_{\max}\!\big(\nabla_\theta^2 \bar M_{\mathcal B}(\theta)\big)
-
\frac{L_{\mathcal B}}{6}\rho^3,
\end{equation}
where $L_{\mathcal B}$ is the Lipschitz constant of $\nabla_\theta^2 \bar M_{\mathcal B}$ on $\mathbb B(\theta,\rho)$ used in the proof of Proposition~\ref{prop:batch-degradation}.
\end{corollary}

\paragraph{Proof of Corollary~\ref{cor:spectral}.}
Let
\[
H_{\mathcal B}(\theta):=\nabla_\theta^2 \bar M_{\mathcal B}(\theta).
\]
Since \(H_{\mathcal B}(\theta)\) is symmetric, the Rayleigh quotient satisfies
\[
v^\top H_{\mathcal B}(\theta) v
\le
\lambda_{\max}\!\bigl(H_{\mathcal B}(\theta)\bigr)
\qquad
\text{for every unit vector } v.
\]
Applying this to \(v=v_{\mathcal B}(\theta)\) gives
\[
v_{\mathcal B}(\theta)^\top
\nabla_\theta^2 \bar M_{\mathcal B}(\theta)
v_{\mathcal B}(\theta)
\le
\lambda_{\max}\!\bigl(\nabla_\theta^2 \bar M_{\mathcal B}(\theta)\bigr).
\]

From Proposition~\ref{prop:batch-degradation}, the equality form gives
\[
\bar M_{\mathcal B}(\theta)-\bar M_{\mathcal B}(\tilde\theta)
=
\rho \|g_{\mathcal B}(\theta)\|_2
-
\frac{\rho^2}{2}
v_{\mathcal B}(\theta)^\top
\nabla_\theta^2 \bar M_{\mathcal B}(\theta)
v_{\mathcal B}(\theta)
+R_{\mathcal B}(\theta,\rho),
\]
with $|R_{\mathcal B}(\theta,\rho)|\le \frac{L_{\mathcal B}}{6}\rho^3$, hence $R_{\mathcal B}(\theta,\rho)\ge -\frac{L_{\mathcal B}}{6}\rho^3$. Substituting the Rayleigh quotient bound into the second-order term yields
\[
\bar M_{\mathcal B}(\theta)-\bar M_{\mathcal B}(\tilde\theta)
\ge
\rho \|g_{\mathcal B}(\theta)\|_2
-
\frac{\rho^2}{2}
\lambda_{\max}\!\bigl(\nabla_\theta^2 \bar M_{\mathcal B}(\theta)\bigr)
-\frac{L_{\mathcal B}}{6}\rho^3,
\]
which proves the claim.
\hfill\(\square\)

\paragraph{Geometric interpretation.}
The expansion should be interpreted as a local directional sensitivity analysis.
The first-order term captures how sensitive the batch-average margin is along the shared batch-induced adversarial probe, while the second-order term captures directional curvature along the same probe.

\subsection{Proof of Proposition~\ref{prop:tri-factor} (Per-Instance Decomposition under the Batch-Wise Anchor)}
\label{app:instance-decomp}

\begin{proof}[Proof of Proposition~\ref{prop:tri-factor}]
For an instance $i\in\mathcal B$, the Anchor Gap under the implemented batch-wise anchor is
\begin{equation}
    \Gamma_i(\theta)
    :=
    M_i(\theta)-M_i(\tilde\theta).
\end{equation}
Assume that $M_i(\theta)$ is three times continuously differentiable. Then
\begin{equation}
    \Gamma_i(\theta)
    =
    \rho \langle \nabla_\theta M_i(\theta), v_{\mathcal B}(\theta)\rangle
    -
    \frac{\rho^2}{2}
    v_{\mathcal B}(\theta)^\top
    \nabla_\theta^2 M_i(\theta)
    v_{\mathcal B}(\theta)
    + O(\rho^3).
\end{equation}

Define the alignment coefficient
\begin{equation}
    \alpha_{i,\mathcal B}(\theta)
    :=
    \frac{
        \langle \nabla_\theta M_i(\theta), g_{\mathcal B}(\theta)\rangle
    }{
        \|\nabla_\theta M_i(\theta)\|_2\,\|g_{\mathcal B}(\theta)\|_2
    }.
\end{equation}
Then
\begin{equation}\label{eq:instance-gap-expansion}
    \Gamma_i(\theta)
    =
    \rho \|\nabla_\theta M_i(\theta)\|_2 \alpha_{i,\mathcal B}(\theta)
    -
    \frac{\rho^2}{2}\kappa_{i,\mathcal B}(\theta)
    + O(\rho^3),
\end{equation}
where
\begin{equation}
    \kappa_{i,\mathcal B}(\theta)
    :=
    v_{\mathcal B}(\theta)^\top
    \nabla_\theta^2 M_i(\theta)
    v_{\mathcal B}(\theta)
\end{equation}
is the directional curvature of instance $i$ along the batch adversarial direction. This is the form stated in Proposition~\ref{prop:tri-factor}.
\end{proof}

\paragraph{Interpretation.}
This decomposition shows that the Anchor Gap is governed by three factors: the local gradient magnitude of the instance, its alignment with the current batch-induced direction, and its directional curvature under the shared adversarial probe. Consequently, the GAPO weight
\begin{equation}
    w_i(\theta)=\beta\,\sigma(\gamma-\beta\Gamma_i(\theta))
\end{equation}
should not be interpreted as an absolute estimate of data quality or intrinsic noise. Rather, $\Gamma_i(\theta)$ measures how much instance $i$'s margin degrades along the current batch direction, and $w_i(\theta)$ acts as an instance-wise weight under the implemented batch-level adversarial probe: instances whose margins degrade substantially along the current batch direction (large $\Gamma_i$) receive smaller updates, while instances stable under the same probe receive larger Anchor-Gap-based weights while preserving their preference-gradient direction.

\section{Experimental Setup and Implementation Details}
\label{app:implementation}

\subsection{Models and Datasets}
\label{app:models-datasets}

We evaluate GAPO on a representative set of open-weight models commonly used in alignment research. Specifically, we utilize Mistral-Instruct-v0.2 (7B), Llama-3-Base/Instruct (8B)~\citep{llama3modelcard}, and Gemma-2-Instruct (9B)~\citep{team2024gemma} as our backbone models.
For the main alignment experiments, we follow the standard protocol of SimPO~\citep{meng2024simpo} and train on the UltraFeedback dataset~\citep{cui2023ultrafeedback}.
To evaluate robustness under noisy preferences, we follow the experimental protocol of Dr.~DPO~\citep{wu2024towards}. We train Pythia-2.8B~\citep{biderman2023pythia} on the Anthropic HH dataset~\citep{bai2022training} and inject synthetic pairwise noise by flipping preference labels at varying rates.

\subsection{Baselines}
\label{app:baselines}
We compare GAPO against a comprehensive suite of offline alignment methods, covering both reference-based and reference-free paradigms:
DPO~\citep{rafailov2023direct}, IPO~\citep{azar2024general}, KTO~\citep{ethayarajh2024kto}, CPO~\citep{xu2024contrastive}, ORPO~\citep{hong2024orpo}, R-DPO~\citep{park2024disentangling}, SimPO~\citep{meng2024simpo} and $\alpha$-DPO~\citep{wu2025alphadpo}.

\subsection{Evaluation Metrics}
\label{app:metrics}
We assess instruction-following capability using AlpacaEval 2.0~\citep{alpaca_eval} and Arena-Hard v0.1~\citep{arenahard2024}, which correlate highly with human judgement.
To evaluate the impact on general reasoning capabilities (the "alignment tax"), we report performance on standard benchmarks including GSM8k~\citep{cobbe2021trainingverifierssolvemath} and ARC-C~\citep{clark2018thinksolvedquestionanswering}.

\subsection{Training Hyperparameters}
We provide the detailed hyperparameters used for training GAPO and baselines across all models in Table~\ref{tab:hyperparams}, following the SimPO setting~\citep{meng2024simpo}. We utilized 2$\times$B200 GPUs for training the 7B/8B/9B models.
\begin{table}[h]
    \centering
    \caption{Hyperparameters for GAPO and baselines.}
    \label{tab:hyperparams}
    \begin{tabular}{l|c c c c }
        \toprule
        \textbf{Hyperparameter} & \textbf{Mistral Inst}& \textbf{Llama-3 Base} & \textbf{Llama-3 Instruct} & \textbf{Gemma-2} \\
        \midrule
        Learning Rate & 5e-7 & 6e-7 & 1e-6& 8e-7 \\
        Global Batch Size & 128 & 128 & 128& 128 \\
        LR Scheduler & Cosine & Cosine & Cosine & Cosine \\
        Warmup Ratio & 0.1 & 0.1 & 0.1& 0.1 \\
        Max Length & 2048 & 2048 & 2048 & 2048 \\
        \midrule
        \textbf{GAPO Specifics} & & & &\\
        Beta ($\beta$) & 2.5 & 2.0 & 2.5& 10 \\
        Gamma ($\gamma$) & 0.1 & 0.5& 0.55 & 0.5 \\
        Perturbation $\rho$ & 0.05 & 0.05 & 0.05 & 0.05 \\
        Anchor Update & Batch-wise & Batch-wise & Batch-wise& Batch-wise \\
        \bottomrule
    \end{tabular}
\end{table}

\subsection{GAPO Pseudo-code}
\label{app:pseudocode}
Algorithm~\ref{alg:gapo} illustrates the PyTorch-style implementation of the Geometric Anchor construction and loss calculation.

\begin{algorithm}[t]
\caption{GAPO Pseudo-code}
\label{alg:gapo}
\begin{lstlisting}[language=Python, basicstyle=\footnotesize\ttfamily, xleftmargin=2em, aboveskip=5pt, belowskip=5pt]
def gapo_loss(model, batch, beta, gamma, rho):
    # 1. Forward Pass & Margin calculation
    logps_w, logps_l = get_log_probs(model, batch)
    margin = (logps_w - logps_l).mean()

    # 2. Compute Adversarial Perturbation
    grads = torch.autograd.grad(-margin, model.parameters())
    gnorm = torch.norm(torch.stack([g.norm() for g in grads]))
    eps = {p: rho * (g / (gnorm + 1e-8)) 
        for p, g in zip(model.parameters(), grads)}

    # 3. Apply Perturbation & Compute Anchor Margin
    with torch.no_grad():
        for p in model.parameters(): p.add_(eps[p])
        anchor_w, anchor_l = get_log_probs(model, batch)
        m_anchor = (anchor_w - anchor_l).detach()
        for p in model.parameters(): p.sub_(eps[p]) # Restore

    # 4. Compute GAPO Loss
    gap = (logps_w - logps_l) - m_anchor
    return -F.logsigmoid(beta * gap - gamma).mean()
\end{lstlisting}
\end{algorithm}

\section{Main Experiment Details and Numerical Results}
\label{app:main-numerics}

\subsection{Noise Robustness: Numerical Results}
\label{app:noise_table}

For reproducibility, Table~\ref{tab:noise_robustness} reports the precise reward-accuracy values plotted in Figure~\ref{fig:noise_robustness}. All values are evaluated on the Pythia-2.8B model trained on the Anthropic HH dataset, following the noisy-preference protocol of Dr.DPO~\citep{wu2024towards}.

\begin{table}[h]
\centering
\renewcommand{\arraystretch}{1.2}
\caption{\textbf{Robustness against label noise: numerical results.} Reward accuracy (\%) of DPO, Dr.DPO, and GAPO under random and length-dependent label flips at five flip rates. Best per row in bold.}
\label{tab:noise_robustness}
\resizebox{0.65\columnwidth}{!}{%
\begin{tabular}{llccccc}
\toprule
Flip & Method & 0\% & 10\% & 20\% & 30\% & 40\% \\
\midrule
\multirow{3}{*}{Random}
 & DPO    & 63.3 & 62.3 & 62.9 & 58.5 & 57.0 \\
 & Dr.DPO & 66.2 & 65.4 & 64.2 & 62.7 & 58.8 \\
 & \textbf{GAPO}  & \textbf{66.8} & \textbf{65.6} & \textbf{64.9} & \textbf{62.8} & \textbf{59.8} \\
\midrule
\multirow{3}{*}{\makecell[l]{Length-\\dependent}}
 & DPO    & 63.3 & 51.7 & 54.3 & 54.3 & \textbf{52.0} \\
 & Dr.DPO & 66.2 & 62.1 & 49.2 & 48.3 & 44.9 \\
 & \textbf{GAPO}  & \textbf{66.8} & \textbf{63.3} & \textbf{58.3} & \textbf{56.6} & \textbf{52.0} \\
\bottomrule
\end{tabular}%
}
\end{table}

\input{tex/appendix_seed_evaluation}

\section{Detailed Mechanistic Statistics for Figure~\ref{fig:mechanism}}
\label{app:mechanism-stats}

This appendix reports detailed statistics underlying
Figure~\ref{fig:mechanism}. We track $2,000$ paired preferences across
five checkpoints spanning the first epoch
($\{0.00,\,0.25,\,0.50,\,0.75,\,1.00\}$ epoch) with $20\%$ random preference flips. The same samples and the
same flip labels are used at every checkpoint, so changes in
separability reflect changes in the model's Anchor-Gap-based reweighting
rather than sample variation.

\paragraph{Weight separation between clean and flipped pairs.}
Figure~\ref{fig:mechanism}a shows that flipped preferences receive
progressively smaller GAPO weights than clean preferences during training.
At the final checkpoint, flipped pairs receive a $16.9\%$ smaller mean
weight on average. Because the GAPO weight is
$w_i = \beta\,\sigma(\gamma - \beta\Gamma_i)$, this indicates that flipped
pairs tend to develop larger Anchor Gaps during training.

\paragraph{Weight-tier flip rates.}
Figure~\ref{fig:mechanism}b stratifies examples directly by the GAPO
weight $w_i$, using the bottom $20\%$, middle $60\%$, and top $20\%$
weight tiers at each checkpoint. The bottom $20\%$ tier corresponds to
the pairs most strongly down-weighted by GAPO, while the top $20\%$ tier
corresponds to the least down-weighted pairs. By the final checkpoint,
the bottom $20\%$ weight tier reaches a flip rate of $30.1\%$, while the
top $20\%$ tier drops to $16.9\%$, yielding a $+13.1$ percentage-point
gap. Thus, the pairs receiving the smallest GAPO weights are
statistically enriched in corrupted supervision in this controlled
setting.

\paragraph{Interpretation and limitation.}
These results support the Anchor-Gap interpretation: GAPO's scalar
reweighting tends to reduce the influence of pairs that are more likely
to be corrupted under controlled random-flip noise. However, the
enrichment is moderate rather than deterministic. Even in the most
down-weighted tier, most pairs are still clean, so the GAPO weight
should not be interpreted as a noise classifier or an absolute
data-quality score. Instead, it is a risk-aware training signal that
statistically biases smaller update weights toward locally brittle
supervision.

\section{Auxiliary Prospective Anchor-Gap Validation}
\label{app:prospective}

\subsection{Subset Selection by Anchor Gap}
\label{app:c1_subset_selection}

Although GAPO is designed for per-instance Anchor-Gap-based reweighting within batch context, the resulting Anchor Gap also induces a per-instance ranking that is informative as a static signal. We do not use this ranking as a standalone data-valuation method, but we report it as evidence that the geometric signal computed at the SFT checkpoint carries useful information about drift-efficient preference learning, complementing the in-training mechanism analysis of Section~\ref{sec:mechanism-empirical}.

\input{tab/main/4_valuation}

\begin{figure}[h]
\centering
\includegraphics[width=0.55\linewidth]{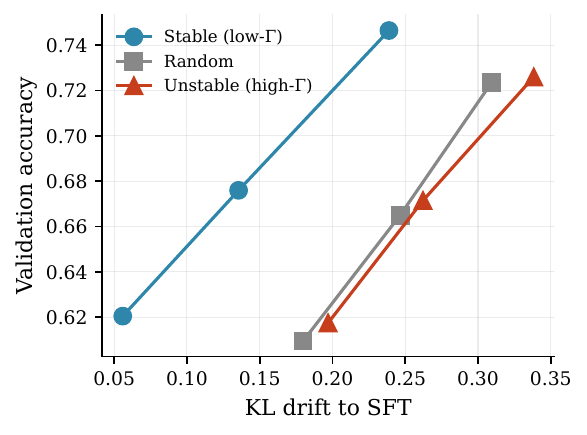}
\caption{\textbf{Prospective validation of the Anchor Gap.} We
compute $\Gamma_i$ once at the SFT checkpoint, select low-$\Gamma$,
random, and high-$\Gamma$ subsets at three data fractions, and train
SimPO on each subset.}
\label{fig:prospective_pareto}
\end{figure}

\paragraph{Setup.}
We sort the training instances by $\Gamma$ and construct three subsets (each $30\%$ of the data):
(1) Stable Subset (Lowest $\Gamma$): geometrically robust instances along the batch direction;
(2) Unstable Subset (Highest $\Gamma$): locally brittle instances under the same probe; and
(3) Random Subset: a baseline.
To avoid circularity (i.e., pruning with GAPO and then re-training GAPO on the pruned data), we train a fixed reference-free learner (SimPO) on each subset under identical hyperparameters. This isolates the effect of Anchor-Gap-based ranking and evaluates whether the induced ordering is informative beyond GAPO's own objective.

\paragraph{Results.}
Table~\ref{tab:data_efficiency} shows a clear ranking: training on the Stable subset yields the best performance, Random is intermediate, and the Unstable subset is the weakest. This indicates that, as an offline ranking signal on a converged checkpoint, instances with large Anchor Gaps tend to be less useful for downstream alignment under SimPO, consistent with their interpretation as locally brittle along the batch direction. We emphasize that this is an auxiliary observation: the main mechanism of GAPO remains training-time Anchor-Gap-based reweighting (Section~\ref{sec:mechanism-empirical}), of which any usefulness as an offline ranking signal is a secondary by-product.

\subsection{Training Dynamics on Anchor-Gap Subsets}
\label{app:fraction_sweep}

To understand \emph{why} the Anchor-Gap-based ranking produces the prospective effect reported in \S\ref{app:c1_subset_selection}, we extend the $f{=}30\%$ experiment of \S\ref{app:c1_subset_selection} to three fractions ($f \in \{30\%, 50\%, 70\%\}$) and analyze training dynamics across all nine fraction--subset combinations.

\paragraph{Full results.} Table~\ref{tab:fraction_sweep} reports validation reward accuracy, validation reward margin, and KL drift to the SFT model for all nine combinations, measured on a validation set of $2{,}000$ preference pairs.

\begin{table}[h]
\centering
\caption{\textbf{Full prospective sweep results.} Validation metrics ($n=2{,}000$ preference pairs) and KL drift from the SFT model, measured after training SimPO on Anchor-Gap-selected subsets. Best accuracy and lowest KL drift per fraction shown in bold.}
\label{tab:fraction_sweep}
\small
\begin{tabular}{lcccc}
\toprule
Subset & Fraction & Val Acc & Val Margin & KL drift \\
\midrule
Stable (low-$\Gamma$)    & 30\% & \textbf{0.621} & 0.043 & \textbf{0.056} \\
Random                   & 30\% & 0.610          & 0.045 & 0.180 \\
Unstable (high-$\Gamma$) & 30\% & 0.618          & 0.060 & 0.197 \\
\midrule
Stable (low-$\Gamma$)    & 50\% & \textbf{0.676} & 0.072 & \textbf{0.135} \\
Random                   & 50\% & 0.665          & 0.084 & 0.247 \\
Unstable (high-$\Gamma$) & 50\% & 0.672          & 0.096 & 0.262 \\
\midrule
Stable (low-$\Gamma$)    & 70\% & \textbf{0.747} & 0.121 & \textbf{0.239} \\
Random                   & 70\% & 0.724          & 0.117 & 0.309 \\
Unstable (high-$\Gamma$) & 70\% & 0.726          & 0.133 & 0.338 \\
\bottomrule
\end{tabular}
\end{table}

The KL drift ordering is strictly monotonic at every fraction (low-$\Gamma$ < random < high-$\Gamma$), and the low-$\Gamma$ subset attains the highest validation accuracy at every fraction. The accuracy advantage grows from $+1.1$pp at $f{=}30\%$ to $+2.3$pp at $f{=}70\%$, while the KL ratio (random/low-$\Gamma$) narrows from $3.22\times$ to $1.30\times$ over the same range.

\paragraph{Training dynamics.} Figure~\ref{fig:appendix_fraction_sweep} provides a comprehensive view across mechanism and outcome metrics. The top row characterizes training-time dynamics: (a) the SFT-time chosen reward $r_w$ is consistently higher (less negative) for low-$\Gamma$ subsets at every fraction, indicating these pairs lie closer to the SFT distribution; (b) the mean training gradient norm is substantially lower for low-$\Gamma$ subsets at every fraction; (c) the rate of large gradient excursions ($\|\nabla\| > 100$) is dramatically higher for high-$\Gamma$ subsets, especially at smaller fractions ($50\%$ of high-$\Gamma$ steps at $f{=}30\%$ vs.\ $0\%$ for low-$\Gamma$); (d) high-$\Gamma$ subsets achieve the largest training-loss reductions, indicating that these pairs do drive strong (but volatile) training signal. The bottom row reports outcomes: (e) validation accuracy is highest for low-$\Gamma$ at every fraction; (f) validation margin is highest for high-$\Gamma$, illustrating that larger training signal translates into larger output margins but not into better classification accuracy; (g) KL drift ranks low-$\Gamma$ < random < high-$\Gamma$ across all fractions; (h) the accuracy-vs-drift scatter summarizes the joint relationship.

\begin{figure}[h]
\centering
\includegraphics[width=\linewidth]{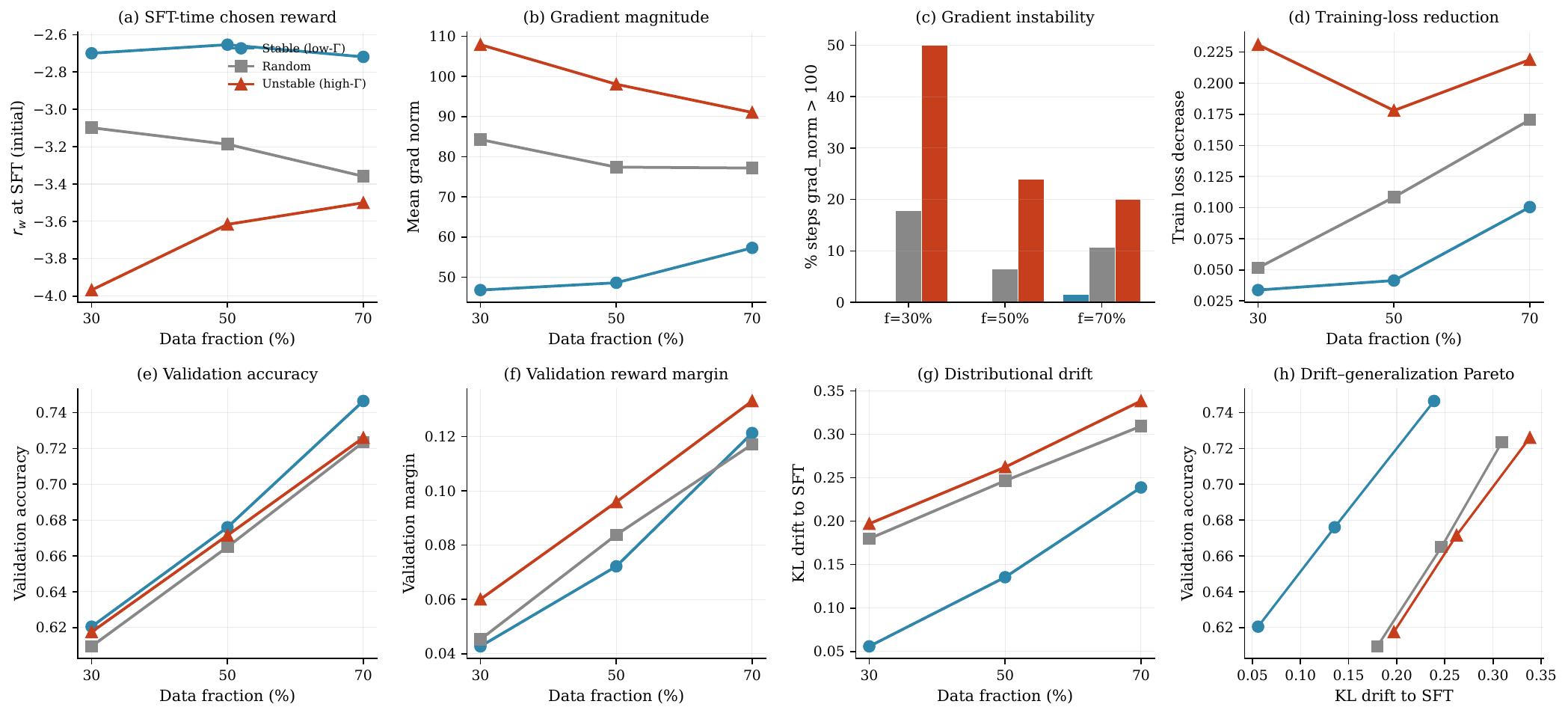}
\caption{\textbf{Training dynamics and outcomes across the fraction--subset grid.} Top row (a--d): training-time mechanism metrics. Bottom row (e--h): post-training outcome metrics. Across all fractions, low-$\Gamma$ subsets exhibit gentler training dynamics yet achieve higher validation accuracy at lower KL drift.}
\label{fig:appendix_fraction_sweep}
\end{figure}

\paragraph{Gradient-norm trajectories.} Figure~\ref{fig:appendix_grad_trajectory} shows full gradient-norm time series during training, smoothed with a window of 3 steps. High-$\Gamma$ subsets exhibit consistently elevated gradient norms throughout training, with isolated extreme excursions (max $\approx 357$ at $f{=}50\%$). Low-$\Gamma$ subsets remain bounded throughout, supporting the interpretation that the Anchor Gap measures local update sensitivity.

\begin{figure}[h]
\centering
\includegraphics[width=\linewidth]{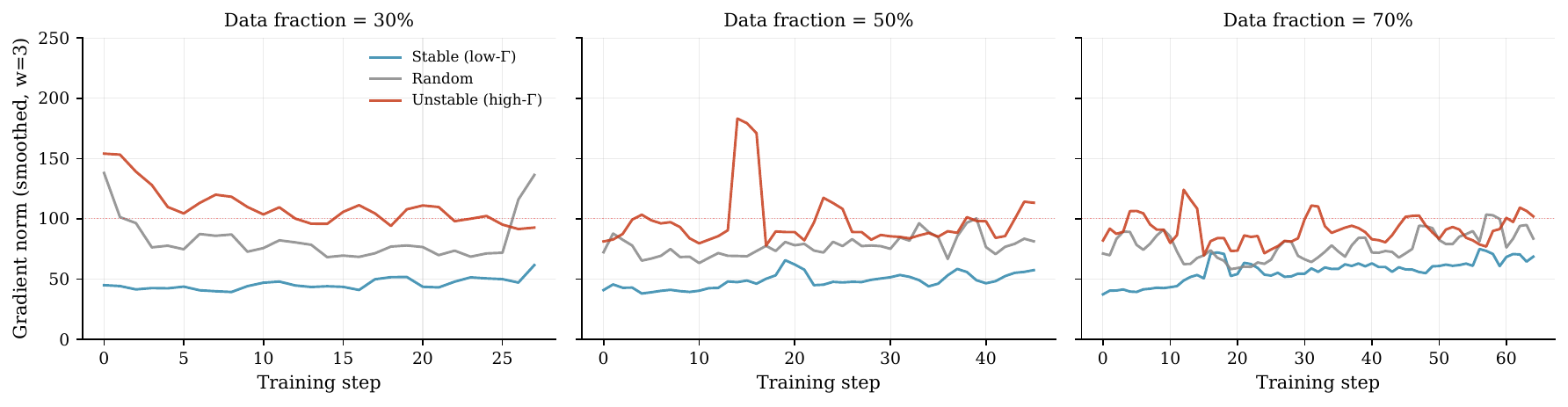}
\caption{\textbf{Gradient-norm trajectories across fractions.} Smoothed gradient norm (window $=3$) during SimPO training on Anchor-Gap-selected subsets. The dotted red line marks $\|\nabla\| = 100$. High-$\Gamma$ subsets exhibit systematically elevated gradients across all fractions, while low-$\Gamma$ subsets remain bounded.}
\label{fig:appendix_grad_trajectory}
\end{figure}

\paragraph{Interpretation.} High-$\Gamma$ subsets are not ``unlearnable'' pairs---they drive the largest training-loss reductions and the largest training reward margins. However, this aggressive training signal does not translate to validation accuracy and comes paired with substantially larger KL drift from the SFT model. Low-$\Gamma$ subsets yield smoother training and more transferable improvements. This pattern, observed across all three fractions, supports interpreting the Anchor Gap as a risk-aware signal for drift-efficient preference learning, complementing the during-training reweighting analyzed in \S\ref{sec:retrospective}.

\section{Batch-Wise Anchor Validation}
\label{app:anchor-validation}

Proposition~\ref{prop:batch-degradation} characterizes how the batch-wise anchor
probes local margin degradation along a shared direction $v_{\mathcal B}$,
rather than along each instance's own worst-case direction
$v_i^{\star} = \nabla_\theta M_i / \|\nabla_\theta M_i\|_2$. A natural
question is whether this batch-shared direction is sufficiently informative
about instance-wise brittleness in practice. We address this empirically.

\paragraph{Setup.}
On a GAPO-trained Mistral-7B-Instruct checkpoint we record
$\alpha_{i,\mathcal B}$ and $\|\nabla_\theta M_i\|_2$ for $1{,}984$
preference pairs distributed over $31$ mini-batches. Using the
first-order expansion in Eq.~\eqref{eq:instance-gap-expansion}, we form a
batch-directional proxy
\[
\widehat{\Gamma}_i^{\text{batch}}
\;:=\; \rho \,\|\nabla_\theta M_i\|_2\, \alpha_{i,\mathcal B},
\]
and an instance-directional counterpart obtained by formally replacing
$v_{\mathcal B}$ with $v_i^{\star}$ in the same expansion,
\[
\widehat{\Gamma}_i^{\text{inst}}
\;:=\; \rho \,\|\nabla_\theta M_i\|_2 .
\]
Their ratio equals $\alpha_{i,\mathcal B}$, isolating the cost of sharing
a single direction across the batch. We emphasize that
$\widehat{\Gamma}_i^{\text{inst}}$ is a first-order proxy, not an executed
instance-wise perturbation.

\paragraph{Findings.}
Across the measured instances, $99.4\%$ satisfy $\alpha_{i,\mathcal B} > 0$,
with median $0.084$. The batch-wise probe therefore moves
in the same half-space as each instance's own first-order worst-case
direction. Despite the modest absolute alignment, the induced \emph{rankings}
agree closely: the Spearman rank correlation between
$\widehat{\Gamma}_i^{\text{batch}}$ and $\widehat{\Gamma}_i^{\text{inst}}$
is $0.909$, and the two rankings share $86.1\%$ of their top-$20\%$ elements
(main Fig.~\ref{fig:batch-validation}). The alignment is driven by a shared
factor: $\|\nabla_\theta M_i\|_2$ exhibits substantially larger
across-sample variation than $\alpha_{i,\mathcal B}$ does (log-scale
standard deviations $0.78$ vs.\ $0.59$).
\paragraph{Interpretation and scope.}
The batch-wise anchor is not an absolute substitute for instance-wise
worst-case perturbation: the two probes differ in both direction and
first-order magnitude. Its practical value lies in \emph{ranking
consistency}: the samples it downweights most strongly are, with high
probability, the same samples an instance-wise probe would downweight.
This supports the batch-shared construction as a computationally
efficient proxy for the ordering-based reweighting in
Eq.~\eqref{eq:GAPO-loss}, while clarifying that the guarantee in
Proposition~\ref{prop:batch-degradation} is a statement about the
batch-directional quantity rather than instance-wise optimality. These
observations are based on first-order approximations on a single
training checkpoint; trajectory-level dynamics of $\alpha_{i,\mathcal B}$
and $\|\nabla_\theta M_i\|$ are analyzed in the main text
Section~\ref{sec:batch-wise-validation}.

\section{Ablations and Additional Comparisons}
\label{app:ablations}

We expand on the summary in Section~\ref{subsec:ablations-summary}. Unless otherwise specified, all ablations are evaluated on AlpacaEval~2.0 with Mistral-Instruct (7B) under our main-experiment hyperparameters.

\subsection{Sensitivity to Perturbation Radius ($\rho$)}
\label{app:ablation_radius}

\input{tab/abl/2_radius}

Table~\ref{tab:ablation_radius} examines the effect of perturbation magnitude. If $\rho$ is too small the anchor may not sufficiently stress-test local geometry; if too large ($\rho \ge 0.075$), the anchor may leave the trusted local neighborhood and introduce additional noise. In our setting, $\rho=0.05$ provides a favorable sensitivity--stability trade-off and is used for our main experiments. We do not claim this value is universally optimal across model families or datasets.

\subsection{Sensitivity to Anchor Batch Multiplier}
\label{app:anchor-batch-size}

GAPO constructs one shared perturbation for each training micro-batch.
The \emph{anchor batch multiplier} controls the size of the micro-batch
used to construct this perturbation. The $1\times$ setting uses a single
preference pair to construct each anchor and therefore corresponds to the
instance-wise endpoint of this design. Increasing the multiplier causes
progressively more preference pairs to share the same perturbation direction.

Intuitively, this parameter controls the granularity of the geometric probe.
A smaller anchor batch produces a more local but potentially noisier
perturbation direction, whereas a larger anchor batch yields a more stable
direction shared across a broader set of preference pairs. Increasing the
batch indefinitely, however, may average away useful local variation in
the preference signal.

\begin{table}[t]
\centering
\caption{
Sensitivity to anchor batch multiplier on Mistral-7B.
The multiplier scales the per-device micro-batch size used to construct
each shared perturbation; gradient-accumulation steps are adjusted
inversely so that the effective optimizer batch size remains fixed.
}
\label{tab:anchor-batch-size}
\begin{tabular}{lccc}
\toprule
Anchor batch multiplier & AlpacaEval 2 WR & GSM8k & ARC-C \\
\midrule
1$\times$  & 35.7 & 36.7 & \textbf{66.8} \\
4$\times$  & 37.2 & 39.0 & 66.2 \\
8$\times$  & \textbf{38.7} & \textbf{40.3} & 66.3 \\
16$\times$ & 36.5 & 38.6 & 66.1 \\
\bottomrule
\end{tabular}
\end{table}

Table~\ref{tab:anchor-batch-size} shows an intermediate optimum.
Moving from the instance-wise $1\times$ setting to shared $4\times$ and
$8\times$ anchors improves both AlpacaEval~2 and GSM8k, with $8\times$
performing best on these metrics. Increasing the multiplier further to
$16\times$ reduces performance. This suggests that sharing a perturbation
across a moderate number of preference pairs provides a more useful probe
than either a fully instance-specific direction or an overly averaged one.

\subsection{Perturbation Scope}
\label{app:ablation_scope}

\input{tab/abl/3_layer}

Table~\ref{tab:ablation_scope} compares perturbing only the last layer (LM head) versus the full model. Full-parameter perturbation performs better in the reported setting, suggesting that preference brittleness is not confined to the final projection layer. We do not claim that the entire representation space is uniformly responsible; rather, restricting the probe to the LM head appears to underestimate where local margin degradation arises in practice.

\subsection{Gradient-Norm Reweighting Baseline}

Proposition~\ref{prop:tri-factor} shows that the Anchor Gap can be decomposed into gradient magnitude, batch-direction alignment, and directional curvature. Moreover, Appendix~F shows that the gradient-magnitude term contributes substantially to the induced ranking. We therefore test whether GAPO's gains can be explained by gradient magnitude alone.

We construct a GradNorm-Reweight baseline by replacing the Anchor Gap $\Gamma_i$ in the GAPO weight
\[
w_i = \beta \sigma(\gamma - \beta \Gamma_i)
\]
with a normalized per-instance margin-gradient norm. This baseline preserves the standard preference-gradient direction and the same scalar reweighting form as GAPO, but removes the batch-conditioned pessimistic anchor and the alignment/curvature factors. All other hyperparameters and training settings are kept identical to the Mistral-Instruct main experiment.

Table~\ref{tab:gradnorm_ablation} reports the results. GradNorm-Reweight is a stronger control than SimPO on some metrics: it improves AlpacaEval 2 LC from 32.1 to 33.5 and GSM8k from 36.6 to 37.0. This confirms that local sensitivity, as measured by gradient magnitude, is relevant for preference reweighting. However, GradNorm-Reweight remains clearly below GAPO, which reaches 35.7 AlpacaEval 2 LC, 38.7 WR, and 40.3 GSM8k. ARC-C remains essentially unchanged across the three methods.

These results support the interpretation of Proposition~\ref{prop:tri-factor}: gradient magnitude is one component of the Anchor Gap, but it is not sufficient to explain GAPO's gains. The batch-conditioned anchor additionally measures whether each instance's margin degrades under the shared pessimistic direction, yielding a more informative stability signal than norm-only attenuation.

\subsection{Distinguishing GAPO from SAM}
\label{app:sam_comparison}

\input{tab/abl/4_sam}

Both GAPO and Sharpness-Aware Minimization (SAM) use parameter perturbations,
but they use the perturbed model in different ways. SAM evaluates the
training-loss gradient at perturbed parameters in order to favor parameters
with lower neighborhood loss. GAPO instead uses the perturbed policy as a
detached diagnostic: it measures how much each pair's preference margin
changes under the shared perturbation and converts this response into an
instance-specific scalar weight. The resulting GAPO update therefore scales
$\nabla_\theta M_i(\theta)$ at the current policy rather than replacing it
with a loss gradient evaluated at the perturbed parameters.

To test whether perturbation-based flatness seeking alone accounts for
GAPO's gains, we train SimPO with SAM using the same perturbation radius
$\rho=0.05$. As shown in Table~\ref{tab:sam_comparison}, SimPO+SAM improves over SimPO
on AlpacaEval~2 and GSM8k, confirming that parameter-space perturbations
can themselves be beneficial. However, the tested SAM configuration remains
below GAPO on AlpacaEval~2, Arena-Hard, and GSM8k. This comparison suggests
that GAPO's benefit is not reproduced simply by applying SAM to the SimPO
objective; using local margin degradation as an instance-specific weighting
signal provides an additional distinction in the tested setting.

\section{Computational Efficiency}
\label{app:efficiency}
\label{app:compute-tradeoff}

GAPO introduces additional computation because each update constructs a pessimistic anchor before evaluating the Anchor Gap.
In our implementation, this requires an additional gradient computation and an anchor forward pass, leading to roughly a \(2\times\) wall-clock overhead relative to standard SimPO under the same Mistral-7B training setup.
We report this cost explicitly because GAPO trades additional per-step computation for more selective Anchor-Gap-based reweighting.

\paragraph{Wall-clock-matched diagnostic.}
To test whether GAPO's gain can be explained merely by spending more wall-clock time, we compare one epoch of GAPO against extended SimPO training under the same recipe.
In our setup, SimPO completes one epoch in approximately 1.5 hours, while one GAPO epoch and two SimPO epochs require approximately 3 hours.
Figure~\ref{fig:efficiency_stability_app} shows that extending SimPO to the comparable wall-clock budget improves AlpacaEval~2 LC only modestly, reaching \(33.0\), whereas one epoch of GAPO reaches \(35.7\).
This \(+2.7\) percentage-point gap suggests that GAPO's improvement is not explained solely by additional wall-clock time in this setting.

\input{tex/appendix_compute_tradeoff}

\paragraph{Scope of the comparison.}
This analysis is a wall-clock diagnostic rather than a complete hardware-independent FLOPs study.
Exact FLOP accounting can depend on implementation details such as fused attention kernels, gradient checkpointing, and how the anchor gradient is computed.
We therefore interpret Figure~\ref{fig:efficiency_stability_app} as evidence that, under our Mistral-7B recipe, simply spending comparable time on longer SimPO training does not reproduce GAPO's AlpacaEval~2 improvement.
A more exhaustive efficiency study across hardware, sequence lengths, and additional baselines is left for future work.

\section{Auxiliary Hessian Eigenspectrum Analysis}
\label{app:hessian}
\input{tab/main/5_hessian}

To complement the theoretical analysis, we measure the local curvature of models trained with DPO, SimPO, and GAPO by approximating the Hessian eigenspectrum via the Lanczos algorithm, reporting eigenvalues ($\lambda$) averaged over training batches.
To reduce computational cost, we estimate the eigenspectrum only for the final projection (LM head) parameters; thus, these results should be interpreted as a targeted probe of local sharpness rather than a full-model curvature estimate.

\textbf{Sharpness.} Table~\ref{tab:hessian_analysis} compares the Hessian eigenvalues.
Large positive eigenvalues (e.g., $\lambda_{max}$) indicate sharp directions, which are often associated with worse generalization~\citep{foret2020sharpness, keskar2017on}.
Across methods, GAPO exhibits a small $\lambda_{max}$ in this analysis, consistent with reduced sharpness in the probed subspace.

\textbf{Subspace curvature profile.} The trend is also reflected in the top-eigenvalue averages. While this does not by itself establish global flatness, it provides empirical support that GAPO's Anchor-Gap-based reweighting can bias optimization toward locally less sharp solutions in the probed subspace.

Notably, the Hessian is often indefinite, exhibiting both positive and negative curvature \citep{dauphin2014identifyingattackingsaddlepoint}, a characteristic feature of saddle points. Under a second-order approximation, the magnitude of the most negative eigenvalue $|\lambda_{\min}|$ characterizes the steepness of the locally most unstable direction.
Compared to DPO/SimPO, GAPO exhibits a substantially smaller-magnitude negative curvature in the LM-head subspace, suggesting reduced local instability along the most negatively-curved direction within the probed subspace. We emphasize that this is a targeted curvature probe and does not by itself establish full-model robustness.

%% file: tex/appendix_seed_evaluation.tex
\subsection{Sensitivity to Random Seeds}
\label{app:seed-evaluation}

We evaluate the effect of random seeds on Mistral-Instruct (7B) under
our main-experiment hyperparameters. GAPO and SimPO are trained with
seeds 1, 2, and 3, keeping each method's configuration fixed.
We report the mean and standard deviation over these three runs.
Seed 42, used during method development, is reported separately.

\begin{table}[!htb]
    \centering
    \caption{Effect of random seeds on Mistral-Instruct (7B).
    We report AlpacaEval~2 length-controlled win rate (LC, \%).
    Mean $\pm$ SD is computed over seeds 1, 2, and 3 using sample
    standard deviation. The development seed (42) is shown separately.
    Bold indicates the higher LC in each column.}
    \label{tab:seed-evaluation}
    \small
    \setlength{\tabcolsep}{7pt}
    \begin{tabular}{lccccc}
        \toprule
        & \multicolumn{4}{c}{Training seeds} & Development seed \\
        \cmidrule(lr){2-5}\cmidrule(lr){6-6}
        Method & Seed 1 & Seed 2 & Seed 3 & Mean $\pm$ SD & Seed 42 \\
        \midrule
        SimPO & 30.44 & 30.88 & 33.01 & $31.44 \pm 1.37$ & 32.10 \\
        GAPO & \textbf{33.89} & \textbf{34.34} & \textbf{35.30}
             & $\mathbf{34.51 \pm 0.72}$ & \textbf{35.70} \\
        \bottomrule
    \end{tabular}
\end{table}

\paragraph{Consistency across seeds.}
Table~\ref{tab:seed-evaluation} shows that GAPO improves over SimPO
on all three additional training seeds. GAPO reaches $34.51 \pm 0.72$ LC,
compared with $31.44 \pm 1.37$ for SimPO, an average gain of
3.07 percentage points. The per-seed gains are 3.45, 3.46, and
2.29 points. Even the lowest GAPO result (33.89) exceeds the highest
SimPO result (33.01), and GAPO exhibits lower variation across runs.
The seed-42 results follow the same ordering (35.70 vs.\ 32.10).

The trend also holds for raw AlpacaEval~2 win rate: over seeds
1, 2, and 3, GAPO reaches $38.19 \pm 0.48$ WR, compared with
$33.83 \pm 1.40$ for SimPO (mean $\pm$ SD), a gain of 4.36 points.
The consistent ordering in both LC and WR supports the stability
of GAPO's improvement across random seeds in this setting.

%% file: tab/main/4_valuation.tex
\begin{table}[t]
    \centering
    \caption{\textbf{Auxiliary SimPO training on Anchor-Gap-ranked subsets.} Comparison of SimPO models trained on 30\% data subsets selected by the Anchor Gap ($\Gamma$). The Stable subset outperforms Random and Unstable.}
    \label{tab:data_efficiency}
    \resizebox{0.7\columnwidth}{!}{%
    \begin{tabular}{lcccc}
    \toprule
    Sampling (30\%) & Metric & AlpacaEval2 & GSM8k & ARC-C \\
    \midrule
    Random & - & 12.8 & 54.3 & 63.6 \\
    Unstable & High $\Gamma$ & 9.80& 54.8 & 63.4 \\
    Stable & Low $\Gamma$ & \textbf{13.8} & \textbf{56.2} & \textbf{64.2} \\
    \bottomrule
    \end{tabular}%
    }
\end{table}

%% file: tab/abl/2_radius.tex
\begin{table}[h]
    \centering
    \caption{\textbf{Effect of Perturbation Radius ($\rho$).} Performance (WR) sensitivity across different perturbation magnitudes on Alpaca-Eval2.}
    \label{tab:ablation_radius}
    \begin{tabular}{lcc}
    \toprule
    \textbf{Radius ($\rho$)} & \textbf{Mistral-Instruct} & \textbf{Llama-3-Base} \\
    \midrule
    0.025 & 37.9 & 20.8 \\
    0.05 & \textbf{38.7} & \textbf{23.7} \\
    0.075 & 37.3 & 20.5 \\
    0.10  & 37.0 & 20.2 \\
    \bottomrule
    \end{tabular}
\end{table}

%% file: tab/abl/3_layer.tex
\begin{table}[h]
    \centering
    \caption{\textbf{Effect of Perturbation Scope.} Comparison between perturbing only the last layer versus full parameters in Llama3-Base.}
    \label{tab:ablation_scope}
    \begin{tabular}{lc}
    \toprule
    Scope & WR \\
    \midrule
    LM head & 20.7 \\
    Full Parameters & \textbf{23.7} \\
    \bottomrule
    \end{tabular}
\end{table}

%% file: tab/abl/4_sam.tex
\begin{table}[h]
    \centering
    \caption{\textbf{Comparison with SimPO+SAM.} SAM introduces optimizer-level flatness-seeking, whereas GAPO uses Anchor-Gap-based instance-wise reweighting. SimPO+SAM improves over SimPO on some metrics but does not match GAPO's overall alignment--reasoning trade-off.}
    \label{tab:sam_comparison}
        \begin{tabular}{lcccc}
        \toprule
        \textbf{Method} & \textbf{AlpacaEval 2} & \textbf{Arena-Hard} &\textbf{GSM8k}&\textbf{ARC-C} \\
        \midrule
        SimPO & 32.1& 21.0&36.6&\textbf{66.3}\\
        SimPO+SAM & 34.9 & 20.1&38.6&66.1 \\
        \textbf{GAPO} & \textbf{35.7} & \textbf{22.7}& \textbf{40.3}&\textbf{66.3}\\
        \bottomrule
        \end{tabular}
\end{table}

%% file: tex/appendix_compute_tradeoff.tex
\paragraph{Runtime comparison across methods.}
Table~\ref{tab:compute-tradeoff} extends the comparison to reference-based
baselines on $2\times$B200 GPUs with Mistral-7B.
The $\alpha$-DPO and R-DPO baselines take approximately 2.33 hours
per epoch, including the reference-model forward pass.

\begin{table}[!htb]
    \centering
    \caption{Wall-clock comparison on Mistral-7B with $2\times$B200 GPUs.
    We report AlpacaEval~2 length-controlled win rate (LC, \%) at the
    corresponding training checkpoints. Rows are grouped by runtime;
    bold indicates the best LC within each group. The 2.33-hour entries
    correspond to one epoch of $\alpha$-DPO and R-DPO.}
    \label{tab:compute-tradeoff}
    \small
    \setlength{\tabcolsep}{10pt}
    \begin{tabular}{lccc}
        \toprule
        Method & Training epochs & Wall-clock (h) & AE2 LC (\%) \\
        \midrule
        SimPO        & 1.00          & 1.50          & \textbf{32.1} \\
        GAPO         & 0.50          & 1.50          & 31.5          \\
        $\alpha$-DPO & $\approx0.64$ & 1.50          & 28.4          \\
        R-DPO        & $\approx0.64$ & 1.50          & 18.8          \\
        \midrule
        $\alpha$-DPO & 1.00          & $\approx2.33$ & \textbf{32.3} \\
        R-DPO        & 1.00          & $\approx2.33$ & 27.3          \\
        \midrule
        SimPO        & 2.00          & 3.00          & 33.0          \\
        GAPO         & 1.00          & 3.00          & \textbf{35.7} \\
        \bottomrule
    \end{tabular}
\end{table}

\paragraph{Performance across budgets.}
Table~\ref{tab:compute-tradeoff} compares methods under fixed training
budgets. At 1.5 hours, half an epoch of GAPO reaches 31.5 LC, nearly
matching one epoch of SimPO (32.1). Under the same budget, GAPO exceeds
$\alpha$-DPO and R-DPO by 3.1 and 12.7 percentage points, respectively.

Extending SimPO from one to two epochs improves LC by 0.9 points.
The comparison keeps the preference dataset fixed: allocating computation
to geometric anchoring yields higher LC than an additional epoch of
standard SimPO training. This highlights a useful trade-off when
additional training compute is available while the preference dataset
is costly to expand.

%% file: tab/main/5_hessian.tex
\begin{table}[h]
    \centering
    \caption{\textbf{Hessian Eigenspectrum Analysis.} Comparison of the Hessian eigenvalues ($\lambda$) averaged over 1,000 samples from the training set.}
    \label{tab:hessian_analysis}
    \resizebox{0.6\columnwidth}{!}{%
    \begin{tabular}{l|ccc}
    \toprule
    \textbf{Method} &$\lambda_{max}$ (Sharpness) $\downarrow$ & Top-5 Avg $\downarrow$  &$\lambda_{min}$ \\
    \midrule
    DPO & 149.6 & 103.4  & -143.1 \\
    SimPO & 158.8 & 98.9  & -106.9 \\
    \midrule
    \textbf{GAPO} & \textbf{145.1} & \textbf{94.9}  & -81.4 \\
    \bottomrule
    \end{tabular} }
\end{table}

%% file: checklist.tex
\section*{NeurIPS Paper Checklist}

\begin{enumerate}

\item {\bf Claims}
    \item[] Question: Do the main claims made in the abstract and introduction accurately reflect the paper's contributions and scope?
    \item[] Answer: \answerYes{}.
    \item[] Justification: The abstract and introduction state the main contributions and scope of GAPO: a batch-conditioned geometric anchor, Anchor-Gap-based conservative reweighting, local theoretical analysis, and empirical evaluation across alignment, reasoning, and noisy-preference settings.

\item {\bf Limitations}
    \item[] Question: Does the paper discuss the limitations of the work performed by the authors?
    \item[] Answer: \answerYes{}
    \item[] Justification: The paper discusses limitations and scope throughout the main text and appendix, including that GAPO is a batch-directional rather than instance-wise worst-case probe, that Anchor Gap should not be interpreted as an absolute data-quality or noise-classification score, that the perturbation radius is not claimed to be universally optimal, and that the Hessian analysis is only a targeted LM-head curvature probe. We also discuss computational overhead in Appendix~H.

\item {\bf Theory assumptions and proofs}
    \item[] Question: For each theoretical result, does the paper provide the full set of assumptions and a complete (and correct) proof?
    \item[] Answer: \answerYes{}
    \item[] Justification: The paper states the smoothness and small-radius assumptions underlying the theoretical analysis in Section~5. Complete derivations and proofs, including the proof of Proposition~5.1, 5.2, the spectral corollary, and the instance-wise Anchor Gap decomposition, are provided in Appendix~A.

    \item {\bf Experimental result reproducibility}
    \item[] Question: Does the paper fully disclose all the information needed to reproduce the main experimental results of the paper to the extent that it affects the main claims and/or conclusions of the paper (regardless of whether the code and data are provided or not)?
    \item[] Answer:  \answerYes{}
    \item[] Justification: The paper provides the main experimental setup in Section~6.1, including models, datasets, baselines, and evaluation metrics. Additional implementation details, hyperparameters, pseudocode, and full numerical results are provided in Appendices~B and~C. These details are intended to make the main results reproducible.

\item {\bf Open access to data and code}
    \item[] Question: Does the paper provide open access to the data and code, with sufficient instructions to faithfully reproduce the main experimental results, as described in supplemental material?
    \item[] Answer: \answerNo{}
    \item[] Justification: The paper uses publicly available models, datasets, and benchmarks, and provides pseudocode and hyperparameters in Appendix~B. However, the current submission does not yet provide an anonymized public code release or executable reproduction scripts. We plan to release code and reproduction instructions upon publication.

\item {\bf Experimental setting/details}
    \item[] Question: Does the paper specify all the training and test details (e.g., data splits, hyperparameters, how they were chosen, type of optimizer) necessary to understand the results?
    \item[] Answer: \answerYes{}
    \item[] Justification: Section~6.1 specifies the models, datasets, baselines, and evaluation metrics. Appendix~B reports training hyperparameters, including learning rates, batch size, scheduler, warmup ratio, maximum sequence length, GAPO-specific parameters, and pseudocode.

\item {\bf Experiment statistical significance}
    \item[] Question: Does the paper report error bars suitably and correctly defined or other appropriate information about the statistical significance of the experiments?
    \item[] Answer: \answerYes{}{}
    \item[] Justification: We report independent replicated experiments in Appendix C.2.

\item {\bf Experiments compute resources}
    \item[] Question: For each experiment, does the paper provide sufficient information on the computer resources (type of compute workers, memory, time of execution) needed to reproduce the experiments?
    \item[] Answer: \answerYes{}
    \item[] Justification: Appendix~B reports the GPU configuration used for the main 7B/8B/9B experiments, and Appendix~H discusses the additional computational overhead of GAPO, including the approximate two-fold wall-clock cost relative to SimPO and wall-clock training dynamics.

\item {\bf Code of ethics}
    \item[] Question: Does the research conducted in the paper conform, in every respect, with the NeurIPS Code of Ethics \url{https://neurips.cc/public/EthicsGuidelines}?
    \item[] Answer: \answerYes{}
    \item[] Justification: The research uses existing public or open-weight models, public preference datasets, and standard evaluation benchmarks. We have reviewed the NeurIPS Code of Ethics and are not aware of any violations.

\item {\bf Broader impacts}
    \item[] Question: Does the paper discuss both potential positive societal impacts and negative societal impacts of the work performed?
    \item[] Answer: \answerYes{}{}
    \item[] Justification: The work studies preference alignment for large language models and may have positive impacts by improving robustness to noisy or brittle preference supervision. It may also have negative impacts if improved alignment or generation quality is applied to harmful, deceptive, or unsafe uses. The paper discusses robustness, noisy supervision, and qualitative failure modes, and the method is intended as a training-time robustness mechanism rather than a deployment-time safety guarantee.

\item {\bf Safeguards}
    \item[] Question: Does the paper describe safeguards that have been put in place for responsible release of data or models that have a high risk for misuse (e.g., pre-trained language models, image generators, or scraped datasets)?
    \item[] Answer: \answerNA{}
    \item[] Justification: The paper does not release a new high-risk pretrained model, image generator, scraped dataset, or other asset requiring controlled access safeguards. GAPO is proposed as an optimization method and is evaluated using existing models and datasets.

\item {\bf Licenses for existing assets}
    \item[] Question: Are the creators or original owners of assets (e.g., code, data, models), used in the paper, properly credited and are the license and terms of use explicitly mentioned and properly respected?
    \item[] Answer: \answerYes{}
    \item[] Justification: The paper cites the existing models, datasets, benchmarks, and methods used in the experiments, including Llama~3, Gemma~2, Pythia, UltraFeedback, Anthropic HH, AlpacaEval~2.0, Arena-Hard, GSM8K, and ARC-C. We use these assets consistently with their intended research use.

\item {\bf New assets}
    \item[] Question: Are new assets introduced in the paper well documented and is the documentation provided alongside the assets?
    \item[] Answer: \answerNA{}
    \item[] Justification: The paper does not release a new dataset, benchmark, pretrained model, or other standalone asset. The main contribution is an optimization method and associated empirical analysis.

\item {\bf Crowdsourcing and research with human subjects}
    \item[] Question: For crowdsourcing experiments and research with human subjects, does the paper include the full text of instructions given to participants and screenshots, if applicable, as well as details about compensation (if any)?
    \item[] Answer: \answerNA{}.
    \item[] Justification: The paper does not conduct new crowdsourcing experiments or new human-subject studies. It uses existing preference datasets and standard automated evaluation benchmarks.

\item {\bf Institutional review board (IRB) approvals or equivalent for research with human subjects}
    \item[] Question: Does the paper describe potential risks incurred by study participants, whether such risks were disclosed to the subjects, and whether Institutional Review Board (IRB) approvals (or an equivalent approval/review based on the requirements of your country or institution) were obtained?
    \item[] Answer: \answerNA{}.
    \item[] Justification: The paper does not conduct new human-subject research or new crowdsourcing experiments requiring IRB or equivalent review.

\item {\bf Declaration of LLM usage}
    \item[] Question: Does the paper describe the usage of LLMs if it is an important, original, or non-standard component of the core methods in this research? Note that if the LLM is used only for writing, editing, or formatting purposes and does \emph{not} impact the core methodology, scientific rigor, or originality of the research, declaration is not required.
    \item[] Answer: \answerNA{}.
    \item[] Justification: The core method development does not use LLMs as an important, original, or non-standard component of the research methodology. Any use of LLMs, if applicable, was limited to writing, editing, or formatting assistance and did not affect the scientific content, methodology, or originality of the work.

\end{enumerate}